\documentclass[sigconf]{aamas}

\usepackage{balance} 

\usepackage{subcaption}
\usepackage{graphicx}
\usepackage{xcolor}

\usepackage{amsmath}
\DeclareMathOperator*{\argmax}{arg\,max}

\usepackage{dsfont}

\usepackage{booktabs}

\usepackage{amsmath,amssymb,amsfonts}
\newcommand{\ind}{\mathbb{I}} 

\usepackage{graphicx}
\usepackage{subcaption} 

\usepackage{booktabs}
\usepackage{multirow}

\theoremstyle{definition}               
\newtheorem{definition}{Definition}[section]  
\newtheorem*{definition*}{Definition}         

\usepackage{amssymb, amsmath}
\usepackage[ruled,vlined,linesnumbered]{algorithm2e}
\usepackage{pgfplots}
\pgfplotsset{compat=1.18}
\usepgfplotslibrary{groupplots}
\usetikzlibrary{patterns,positioning}
\usepackage{booktabs}
\usepackage{multirow}
\usepackage{xcolor}
\usepackage{colortbl}
\usepackage{array}
\usepackage{booktabs}
\usepackage{tabularx}
\usepackage{array}
\usepackage{makecell}
\usepackage{needspace}
\usepackage{placeins}
\usepackage{caption}
\usepackage{cuted}

\definecolor{cellS}{HTML}{4472C4}   
\definecolor{cellH}{HTML}{595959}   
\definecolor{cellG}{HTML}{548235}   
\definecolor{cellF}{HTML}{F2F2F2}   

\tikzset{
    block/.style={
        rectangle,
        draw,
        rounded corners,
        minimum height=0.9cm,
        minimum width=2.2cm,
        align=center,
        font=\footnotesize
    },
    process/.style={
        block,
        fill=gray!10
    },
    highlight/.style={
        block,
        fill=blue!10,
        draw=blue!60
    },
    data/.style={
        block,
        fill=green!10
    },
    arrow/.style={
        ->,
        thick
    }
}
\usetikzlibrary{
  arrows.meta,
  positioning,
  shapes.geometric,
  shapes.misc,
  fit,
  backgrounds,
  calc,
  decorations.pathreplacing,
  decorations.markings,
  patterns,
  shadows.blur,
}

\DontPrintSemicolon
\SetAlFnt{\small}
\SetAlgoNlRelativeSize{-1}
\SetNlSty{\tiny}{}{:}
\SetAlgoInsideSkip{smallskip}

\SetKwSty{textbf}{}{}                 
\SetCommentSty{\footnotesize\upshape} 
\SetKwComment{tcp}{\(\triangleright\)\ }{} 

\DeclareMathOperator{\MaxLID}{MaxLID}

\DontPrintSemicolon
\SetAlFnt{\small}

\SetKwSty{AlgoKwSty}{}{} 

\SetCommentSty{AlgoCommentSty} 

\SetNlSty{AlgoNlSty}{}{:}      

\newcommand{\AlgoCommentMark}{\ensuremath{\triangleright}}
\SetKwComment{tcp}{\AlgoCommentMark\ }{}

\definecolor{cSafe}{HTML}{2E8B57}
\definecolor{cUnsafe}{HTML}{C0392B}
\definecolor{cNeutral}{HTML}{2C3E50}
\definecolor{cTask1}{HTML}{2980B9}
\definecolor{cTask2}{HTML}{E67E22}
\definecolor{cMethod}{HTML}{6C3483}
\definecolor{cGrayBg}{HTML}{F5F5F5}
\definecolor{cGrayMed}{HTML}{BDC3C7}
\definecolor{cGrayDark}{HTML}{7F8C8D}

\usepackage{pifont}

\newtheorem{proposition}{Proposition}
\newtheorem{assumption}{Assumption}
\newtheorem{remark}{Remark}
\newtheorem{corollary}{Corollary}
\newtheorem{theorem}{Theorem}

\newcommand{\Asafe}{\mathcal{A}^{\mathrm{safe}}}
\newcommand{\phisafe}{\phi^{\mathrm{safe}}}

\setcopyright{none}
\renewcommand\footnotetextcopyrightpermission[1]{}

\acmSubmissionID{22}

\title{\textsc{SafeAdapt}: Provably Safe Policy Updates \\ in Deep Reinforcement Learning}

\author{Maksim Anisimov}
\affiliation{
  \institution{Imperial College London}
  \city{London}
  \country{United Kingdom}}
\email{m.anisimov23@imperial.ac.uk}

\author{Francesco Belardinelli}
\affiliation{
  \institution{Imperial College London}
  \city{London}
  \country{United Kingdom}}
\email{francesco.belardinelli@imperial.ac.uk}

\author{Matthew Wicker}
\affiliation{
  \institution{Imperial College London}
  \city{London}
  \country{United Kingdom}}
\email{m.wicker@imperial.ac.uk}

\begin{abstract}
Safety guarantees are a prerequisite to the deployment of reinforcement learning (RL) agents in safety-critical tasks. Often, deployment environments exhibit non-stationary dynamics or are subject to changing performance goals, requiring updates to the learned policy. This leads to a fundamental challenge: how to update an RL policy while preserving its safety properties on previously encountered tasks? The majority of current approaches either do not provide formal guarantees or verify policy safety only \emph{a posteriori}. We propose a novel \emph{a priori} approach to safe policy updates in continual RL by introducing the \emph{Rashomon set}: a region in policy parameter space certified to meet safety constraints within the demonstration data distribution. We then show that one can provide formal, provable guarantees for arbitrary RL algorithms used to update a policy by projecting their updates onto the Rashomon set. Empirically, we validate this approach across grid-world navigation environments (Frozen Lake and Poisoned Apple) where we guarantee an \emph{a priori} provably deterministic safety on the source task during downstream adaptation. In contrast, we observe that regularisation-based baselines experience catastrophic forgetting of safety constraints while our approach enables strong adaptation with provable guarantees that safety is preserved.
\end{abstract}


\newcommand{\BibTeX}{\rm B\kern-.05em{\sc i\kern-.025em b}\kern-.08em\TeX}

\begin{document}


\pagestyle{fancy}
\fancyhead{}


\maketitle


\section{Introduction}\label{sec:intro}

Reinforcement learning (RL) has achieved remarkable success in sequential decision-making, from game playing \cite{DBLP:journals/corr/MnihKSGAWR13}
to robotic control \cite{SurveyRobotics} and autonomous driving \cite{KHJMRALBS19}. As RL agents move closer to real-world deployment in safety-critical domains -- including autonomous vehicles, medical treatment planning, and industrial process control -- it becomes essential to guarantee that learned policies satisfy safety constraints \citep{bommasani2021opportunities, szpruch2025insuring}. At the same time, deployed agents must often operate in non-stationary settings: objectives evolve, dynamics drift, and new tasks appear. This continual adaptation introduces a central tension: policy updates aimed at improving performance often substantially degrade safety properties on previously encountered tasks.

Existing approaches to this problem fall short in important ways. Regularisation-based continual learning methods such as Elastic Weight Consolidation (EWC)~\citep{kirkpatrick2017overcoming} discourage parameter drift but do not provide formal guarantees that safety is preserved after adaptation. Shielding and action-space filtering methods~\citep{alshiekh2018safe,dalal2018safe} can guarantee safety by overriding unsafe actions online, even after unrestricted policy adaptation. However, they do not certify that the adapted policy itself remains safe, they rely on a runtime intervention mechanism, and they may mask catastrophic forgetting of safe behaviour rather than prevent it. Safety-aware transfer and safe policy-update methods~\citep{held2017probabilistically,cloete2025sport} typically provide probabilistic guarantees via policy-ratio constraints or constrained optimisation, but they are unable to provide \textit{a priori} guarantees of safety and instead require users to perform expensive evaluations to verify that updated policy satisfies safety. 
More broadly, safe RL methods that ensure safety during training on a single task~\citep{achiam2017cpo_icml,chow2018lyapunov} do not directly address what happens when the resulting policy is later adapted to a new task. We provide the first, to our knowledge, \emph{parameter-space certificate for source-task safety} during downstream policy adaptation in RL. In finite or discretised settings with an evaluable unsafety labelling function and greedy deployment, this certificate becomes \emph{deterministic} and does not require full knowledge of transition dynamics.

\paragraph{Contributions.}
We propose \textsc{SafeAdapt}---an \emph{a priori} approach to provably safe policy updates in continual RL. Our method leverages the Local Invariant Domain (LID) framework~\citep{elmeckerplakolm2025provably} to construct a \emph{certified} region in policy parameter space -- a \emph{Rashomon set} -- within which all policies satisfy a specified safety property on the source task. Concretely, given an unsafety labelling function that marks unsafe state--action pairs, we:
\begin{itemize}
    \item \textbf{Formulate the notion of a \textit{Rashomon Set} in policy parameter space} and show that computing a set of safe parameters can be posed as an optimisation problem using a sound, differentiable \textit{safety surrogate}.
    \item \textbf{Compute certified Rashomon sets in parameter space} using Interval Bound Propagation (IBP) and our novel differentiable safety surrogate, ensuring that every policy inside the set satisfies the source-task safety specification under greedy deployment.
    \item \textbf{Perform downstream adaptation with constraints} by combining PPO updates with projected gradient descent, guaranteeing that parameter updates remain inside the certified safe region throughout training.
\end{itemize}
 
Empirically, we evaluate the method across grid-world navigation environments (Poisoned Apple and Frozen Lake). Our Rashomon-constrained updates preserve source-task safety after adaptation while achieving competitive downstream performance (see Section~\ref{sec:experiments}). Notably, we show that continual learning methods without formal guarantees (e.g. EWC) can lead to catastrophic forgetting of safe behaviour in the source task while our methods are able to at once formally rule out this catastrophic forgetting while enabling strong adaptation to new tasks. 

\paragraph{Paper organisation.}
This paper is organised as follows. Section~\ref{sec:related-work} reviews related work on safe reinforcement learning, continual learning, and neural network verification. Section~\ref{sec:prelims} provides background on MDPs with unsafe states, policy optimisation, and interval bound propagation. Section~\ref{sec:methodology} presents the proposed method---\textsc{SafeAdapt}. Section~\ref{sec:experiments} describes the experimental setup and results\footnote{Code is available at \url{https://github.com/maxanisimov/provably-safe-policy-updates}.}. Section~\ref{sec:conclusion} discusses implications, limitations, and future work.

\section{Related Work}\label{sec:related-work}

Our method lies at the intersection of safe reinforcement learning, continual learning, and neural network verification. We focus on \emph{formally} certifying preservation of source-task safety during downstream adaptation in \emph{parameter space} in the presence of an \emph{unsafety labelling function}. We also note that when safe actions deterministically preserve safety under greedy deployment, this certification becomes deterministic. Table~\ref{tab:related-works-comparison} summarises the key distinctions between our algorithm and related methods.

\begin{table}[ht]
\centering
\caption{Comparison of safety-aware adaptation methods. \textbf{Safety structure}: the type of prior safety/environment structure assumed by the method. \textbf{Formal?}: whether the method can provide a formal safety guarantee. \textbf{Space}: where the safety constraint is enforced. \textbf{CL}: whether the method addresses continual learning.}
\label{tab:related-works-comparison}
\setlength{\tabcolsep}{3pt}
\scriptsize
\begin{tabular}{@{}lcccc@{}}
\toprule
\textbf{Method} & \textbf{Safety structure} & \textbf{Formal?} & \textbf{Space} & \textbf{CL} \\
\midrule
EWC~\citep{kirkpatrick2017overcoming}           & None beyond source-task data & No  & Parameters   & Yes \\
Shielding~\citep{alshiekh2018safe}              & Abstract environment model & Yes & Action & No \\
SPoRt~\citep{cloete2025sport}                   & Base policy + scenario-based safety data      & Yes & Policy ratio & No \\
\citet{zhang2024safety}                         & Explicit transition dynamics model            & Yes & Action/state & No \\
SaGui~\citep{yang2023sagui}                     & Safety guidance / exploration constraints     & No  & Action       & No \\
\citet{held2017probabilistically}               & Learned damage model                          & Yes & Action       & No \\
\citet{berkenkamp2017safe}                      & Learned dynamics model                        & Yes & State        & No \\
P\&C~\citep{schwarz2018progress}                & Transfer mechanism & No & Distillation & Yes \\
\midrule
\textbf{\textsc{SafeAdapt} (Ours)}              & \textbf{Unsafety labelling function $U(s,a)$} & \textbf{Yes} & \textbf{Parameters} & \textbf{Yes} \\
\bottomrule
\end{tabular}
\end{table}

\paragraph{Safe Reinforcement Learning.}\label{sec:safe_rl}
Safe RL is commonly formalised via Constrained MDPs~\citep{altman1999cmdp}---see~\citet{garcia2015comprehensive,gu2022review} for surveys. \emph{Constrained policy optimisation} methods such as CPO~\citep{achiam2017cpo_icml}, CRPO~\citep{xu2021crpo}, and PCPO~\citep{yang2020projection} guarantee near-constraint satisfaction during training but only for \emph{individual} tasks; they do not address safety preservation when the policy is adapted to a new task. \emph{Lyapunov-based} approaches provide safe updates at every iteration: \citet{berkenkamp2017safe} require a learned dynamics model (Gaussian processes), while \citet{chow2018lyapunov} formulate Lyapunov constraints within a model-free primal-dual framework and operate in value space rather than parameter space. \emph{Shielding} methods~\citep{alshiekh2018safe,dalal2018safe,elsayedaly2021safe} correct unsafe actions at runtime via safety filters, typically requiring an abstract environment model for shield synthesis. In contrast, our method certifies safety in parameter space, thereby enforcing implicit safety of a neural policy which does not require any runtime intervention. 

\paragraph{Continual Learning.}\label{sec:continual_learning}
Continual learning addresses catastrophic forgetting~\citep{ring1994continual,thrun1998lifelong} via the stability--plasticity trade-off. Regularisation methods such as EWC~\citep{kirkpatrick2017overcoming}, SI~\citep{zenke2017continual}, and MAS~\citep{aljundi2018memory} penalise changes to important parameters but provide no formal guarantees on what knowledge the neural network retains. EWC is the most natural baseline for our approach: both methods operate in parameter space and aim to limit how far parameters move from a reference policy. However, EWC's soft quadratic penalty only \emph{encourages} proximity to prior parameters -- a sufficiently strong reward signal can override the penalty and cause arbitrary safety degradation. In contrast, our method enforces a \emph{hard} constraint: parameters are projected onto the certified safe orthotope after every gradient step, providing a formal guarantee that no update can violate safety in a source task. Architecture and replay-based methods~\citep{rusu2016progressive,mallya2018packnet,lopezpaz2017gradient} address forgetting through structural or data-level mechanisms and are orthogonal to our parameter-space certification; in principle, they could be combined with our approach. In RL setting, continual methods such as Progress~\&~Compress~\citep{schwarz2018progress} and policy distillation~\citep{rusu2015policy} mitigate forgetting via knowledge transfer but provide no formal safety guarantees. Existing continual RL methods~\citep{khetarpal2022towards}
do not provide formal guarantees on source-task safety preservation, and our paper addresses this gap.

\paragraph{Safety-Aware Policy Updates.}\label{sec:safe_transfer}
A growing literature considers safety preservation during policy updates. SPoRt~\citep{cloete2025sport} bounds the violation probability of a task-specific policy using the scenario approach~\citep{campi2008exact} and constrains the policy ratio $\pi_{\mathrm{task}}(a|s) / \pi_{\mathrm{base}}(a|s) \leq \alpha$ via a per-timestep convex projection. Its guarantees are probabilistic and the bound grows with episode length~$T$, becoming vacuous for long horizons. Importantly, SPoRt addresses safe adaptation of a base policy to a single task -- it does not consider catastrophic forgetting of safety requirements on prior tasks. \citet{held2017probabilistically} formalise safe transfer via an expected damage bound, providing probabilistic guarantees that require a learned damage model. SaGui~\citep{yang2023sagui} provides empirical safety improvements without formal certificates. \citet{zhang2024safety} introduce a transfer-learning framework for safe RL that trains in a non-dangerous environment and then transfers the policy to the dangerous target system with theoretical stability and safety guarantees. Finally, \citet{bouammar2015safe} propose a lifelong policy-gradient method that learns multiple tasks online while enforcing safety constraints and achieving sublinear regret, but without any formal certificates.

\paragraph{Neural Network Verification.}\label{sec:verification}
We leverage neural network verification to certify safety over parameter regions. Interval Bound Propagation (IBP)~\citep{gowal2018effectiveness,mirman2018differentiable} computes output bounds via a single forward pass but produces increasingly loose bounds with network depth, limiting the size of the certified region. 
Tighter methods such as CROWN~\citep{zhang2018crown} and $\alpha$/$\beta$-CROWN~\citep{xu2021fast} can be substituted at higher computational cost.
While these methods focus on certifying input regions, work in Bayesian Neural Networks (BNN) introduced the notion of certifying parameter regions \citep{wicker2020probabilistic} which has also been completed for BNN policies in RL \citep{wicker2021certification, wicker2024probabilistic}. Certified parameter regions were further optimised and studied by computing bounds over their derivatives \citep{wickerrobust} to proving dataset  robustness \citep{sosnincertified} and privacy \citep{sosnin2025abstract, wicker2025certification}.
Most closely related is the Local Invariant Domain (LID) framework of \citet{elmeckerplakolm2025provably}, which computes maximal certified regions in parameter space via primal--dual optimisation over abstract domains. We adopt the LID framework and extend it to safe continual RL, where the specification is defined by an unsafety labelling function over state--action pairs and multiple safe actions per state are handled via a multi-label certificate. Prior verification in RL~\citep{bastani2018verifiable} has been used as a post-hoc check; our work uses it to define a safe region \emph{within which} continual learning is allowed to proceed.

\paragraph{\textsc{SafeAdapt} positioning.}\label{sec:positioning}
Our proposed method \textsc{SafeAdapt} assumes access to an unsafety labelling function $U: \mathcal{S} \times \mathcal{A} \to \{0, 1\}$. This is motivated by the observation that engineers can often specify forbidden state-action pairs even when the transition dynamics is unknown. Using the LID framework~\citep{elmeckerplakolm2025provably}, we construct a certified safe parameter region and constrain all downstream updates to remain within it via projected gradient descent. This yields, to our knowledge, the first method that can provide formal source-task safety guarantees during downstream adaptation.

\section{Preliminaries}\label{sec:prelims}

We introduce the background required for our method: Markov Decision Processes with unsafe states and policy optimisation via reinforcement learning. We refer to~\citet{SuttonBarto2018} for broader background on RL.

\paragraph{MDPs with Unsafe States.}\label{sec:prelim_mdp}
We model sequential decision-making as a Markov Decision Process (MDP) $M = (\mathcal{S}, \mathcal{A}, P, r, \gamma, \mu_0)$ with state space $\mathcal{S}$, action space $\mathcal{A}$, transition dynamics $P(s' | s, a)$, reward function $r: \mathcal{S} \times \mathcal{A} \to \mathbb{R}$, discount factor $\gamma \in [0,1)$, and initial state distribution $\mu_0$. Since not all actions may be executable in every state, we also define the set of state-specific actions as $\mathcal{A}(s)$. To capture safety, the state space is partitioned into safe and \emph{unsafe regions}, which is common in safe RL literature \citep{gu2022review}. Intuitively, the set of \emph{unsafe states} $\mathcal{S}_u \subseteq \mathcal{S}$ is the set of states in which the agent is experiencing a safety violation. For example, those can be hazards such as the ice holes in the Frozen Lake environment, which is used in our experiments.

\begin{definition}[Unsafety labelling function]
The {\em unsafety labelling function} $U \colon \mathcal{S} \times \mathcal{A} \to \{0, 1\}$
assigns $U(s, a) = 1$ iff $P\left(s' \in \mathcal{S}_u | s, a \right) > 0$, that is, if taking action $a$ in state $s$ can lead to an unsafe state in the next step, and $U(s, a) = 0$ otherwise.
\end{definition}

Note that we use a definition of $U$ which is \emph{conservative} since it does not allow \emph{any} chance that the next state will be unsafe. Using $U$, we define the  \emph{set of safety-critical states} and \emph{safe action set}:
\begin{definition}[Safety-critical states]
The set of {\em safety-critical states} is
$\mathcal{S}_{\mathrm{sc}} = \{s \in \mathcal{S} : \exists\, a \in \mathcal{A}(s) \text{ s.t. } U(s, a) = 1\}$,
i.e., states at which at least one action is unsafe.
\end{definition}

\begin{definition}[Safe action set]\label{def:safe_action_set}
    The \emph{safe action set} for a state $s$ is a set of actions which 
    guarantee that the next state will be safe, i.e.:
    \begin{equation}\label{eq:safe_action_set}
    \mathcal{A}^{\mathrm{safe}}(s) = \{a \in \mathcal{A}(s) : U(s, a) = 0\}.
\end{equation}
\end{definition}
Finally, we introduce safe policies as those that return safe actions in every safety-critical state. A policy $\pi_\theta$, parametrised by $\theta$, is a function that maps states into action distributions $\pi_\theta(a | s)$. 
\begin{definition}[Safe policy]
    A policy $\pi$ is \emph{safe} if its deterministic (i.e. greedy) deployment leads to selecting safe actions in every state where unsafe actions exist, i.e.:
    \begin{equation}\label{eq:safe_policy}
        \argmax_{a}\pi(a \mid s) \in \mathcal{A}^{\mathrm{safe}}(s) \;
        \forall s \in \mathcal{S}_{\mathrm{sc}}. 
    \end{equation}
\end{definition}

\paragraph{Reinforcement Learning.}
Provided an RL problem modelled as an MDP, the standard RL objective is to find a policy maximising the expected discounted return defined as:
\begin{equation}\label{eq:rl_objective}
    J(\theta) = \mathbb{E}_{\tau \sim \pi_\theta}\!\left[\sum_{t=0}^{\infty} \gamma^t\, r(s_t, a_t)\right],
\end{equation}
where the expectation is over trajectories $\tau = (s_0, a_0, s_1, a_1, \ldots)$ with $s_0 \sim \mu_0$, $a_t \sim \pi_\theta(\cdot \mid s_t)$, and $s_{t+1} \sim P(\cdot \mid s_t, a_t)$.

\paragraph{Problem statement.}
Given a source-task policy $\pi_{\theta_{\mathrm{source}}}$ trained on an
MDP $\mathcal{M}_{\mathrm{source}}$, an unsafety labelling function
$U_{\mathrm{source}} \colon \mathcal{S} \times \mathcal{A} \to \{0, 1\}$
for $\mathcal{M}_{\mathrm{source}}$, and a downstream MDP $\mathcal{M}_{\mathrm{down}}$ defined over the same state--action space but different transition dynamics, we seek an adapted policy
$\pi_{\theta_{\mathrm{down}}}$ that maximises the downstream return
$J_{\mathrm{down}}(\theta_{\mathrm{down}})$ while provably remaining safe on the
source task, i.e.,
\begin{equation}
    \forall\, s \in \mathcal{S}_{\mathrm{sc}}:\quad
    \argmax_{a \in \mathcal{A}(s)}\, \pi_{\theta_{\mathrm{down}}}(a \mid s)
    \;\in\; \mathcal{A}^{\mathrm{safe}}(s).
    \label{eq:problem-statement}
\end{equation}

We require this guarantee to hold \emph{a priori} --- as a certified property of the policy parameters -- without full access to the source-task transition dynamics and without runtime action filtering.

\begin{definition}[Rashomon set]\label{def:rashomon_set}
     We define a \emph{Rashomon set}~\citep{semenova2019rashomon} $\Theta^{\mathrm{safe}}_{i}$ for a task $i$ as a parameter space in which any policy $\pi_{\theta'}$ parameterised with $\theta' \in \Theta^{\mathrm{safe}}_{i}$ is safe.
\end{definition}

\section{Methodology}\label{sec:methodology}

We present our method \textsc{SafeAdapt}
for safe continual reinforcement learning.
The key idea is to construct a certified region in parameter space -- a \emph{Rashomon set}~\citep{elmeckerplakolm2025provably} -- around a safe source-task policy, and to constrain all downstream policy updates to remain within this region via projected gradient descent.
We first present the algorithm (\S\ref{sec:algorithm}), describe each phase (\S\ref{sec:phase1}--\S\ref{sec:phase4}), and then state the assumptions and theoretical guarantees (\S\ref{sec:assumptions}).

\subsection{Algorithm Overview}\label{sec:algorithm}

Given a source-task policy $\pi_{\mathrm{source}}$ and an unsafety labelling function $U_{\mathrm{source}}$, we first construct a certified safe region $\Theta_{\mathrm{source}}^{\mathrm{safe}}$ (a Rashomon set) in parameter space within which all policies provably satisfy the safety specification in the source task. Downstream adaptation then proceeds as a constrained optimisation: any gradient-based method may be used, provided that after each update the parameters are projected back onto $\Theta_{\mathrm{source}}^{\mathrm{safe}}$ via element-wise clipping.

Algorithm~\ref{alg:safe_continual_rl} gives the complete procedure and Figure~\ref{fig:method_diagram} illustrates its structure. The method proceeds in three phases: (1)~construct a safe demonstration dataset, (2)~compute the maximal locally invariant domain (LID), and (3)~adapt to the downstream task with projected gradient descent.

\begin{algorithm}[t]
\caption{\textsc{SafeAdapt}: Safe Continual RL via Rashomon Set Computation}
\label{alg:safe_continual_rl}
\KwIn{\;
(1) Source-task policy $\pi_{\mathrm{source}}$ with parameters $\theta_{\mathrm{source}}$ \;
(2) Unsafety labeller $U_{\mathrm{source}}:\mathcal{S}\times\mathcal{A}\to\{0,1\}$\;
(3) Global safety specification $\Phi^{\mathrm{safe}}$ and its surrogate $\widetilde{\Phi}^{\mathrm{safe}}$\;
(4) Global safety surrogate threshold $\delta$\;}
\KwOut{Policy $\pi_{\theta_{\mathrm{down}}}$ adapted to a downstream task that is at least as safe as $\pi_{\mathrm{source}}$ in the source task\;}
\BlankLine
\tcc{\textbf{Phase 1: Safe behaviour demonstration (\S\ref{sec:phase1})}}
$S_{\mathrm{sc}} \leftarrow \{\, s \in S : \exists a \in \mathcal{A}(s)\ \text{s.t.}\ U_{\mathrm{source}}(s,a)=1 \,\}$\;
$\mathcal{A}^{\mathrm{safe}}(s) \leftarrow \{\, a \in \mathcal{A}(s) : U_{\mathrm{source}}(s,a)=0 \,\}$\;
$D^{\mathrm{safe}}_{\mathrm{source}} \leftarrow \{\, (s, \mathcal{A}^{\mathrm{safe}}(s)) : s \in \mathcal{S}_{\mathrm{sc}} \,\}$
\BlankLine
\tcc{\textbf{Phase 2: Certified safe region (\S\ref{sec:phase3})}}
\If{$\widetilde{\Phi}^{\mathrm{safe}}_{\tau, \mathrm{sc}}(\pi_{\mathrm{source}};\, D_{\mathrm{source}}^{\mathrm{safe}}) < \delta$}{
  $\Theta_{\mathrm{source}}^{\mathrm{safe}} \gets \emptyset$\;
  \Return{$\bot$} \tcp*{method cannot compute a provably safe parameter region}
}
$\alpha^\star \gets \MaxLID(\theta_{\mathrm{source}},\,\widetilde{\Phi}^{\mathrm{safe}}_{\tau, \mathrm{sc}},\,\delta)$ \tcp*{primal--dual + IBP}
$\Theta_{\mathrm{source}}^{\mathrm{safe}} \gets
\{\theta': \theta_{\mathrm{source}}-\alpha^\star \le \theta' \le \theta_{\mathrm{source}}+\alpha^\star\}$\tcp*{Rashomon set}
\BlankLine
\tcc{\textbf{Phase 3: Safe downstream adaptation (\S\ref{sec:phase4})}}
(a) Initialise the downstream policy as a source-task policy: $\pi_{\theta_{\mathrm{down}}} \gets \pi_{\mathrm{source}}$\;
(b) Solve via a projected gradient descent:
\begin{equation}
    \theta_{\mathrm{down}} = \argmax_{\theta \,\in\, \Theta_{\mathrm{source}}^{\mathrm{safe}}} \; J_{\mathrm{down}}(\theta),
\end{equation}
\Return{$\pi_{\theta_{\mathrm{down}}}$}\;
\end{algorithm}

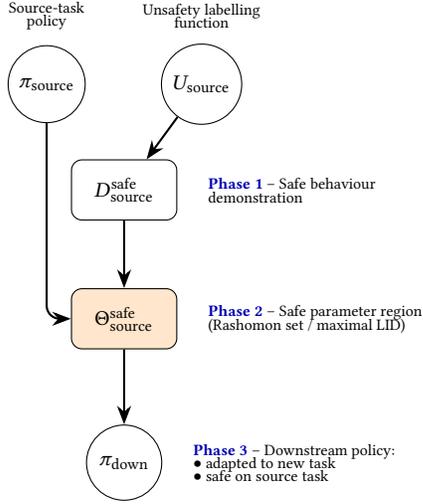
\begin{figure}[t]
\centering
\begin{tikzpicture}[
    node distance=1cm and 1.2cm,
    box/.style={draw, rounded corners, minimum width=1.4cm, minimum height=0.8cm, align=center, font=\small},
    circlebox/.style={draw, circle, minimum size=1cm, align=center, font=\small},
    arr/.style={-{Stealth[length=2.5mm]}, thick},
    phase/.style={font=\scriptsize\bfseries, text=blue!70!black},
]

\node[circlebox] (pi1) {$\pi_{\text{source}}$};
\node[circlebox, right=1cm of pi1] (U1) {$U_{\text{source}}$};
\node[above=0.1cm of pi1, font=\scriptsize, align=center] {Source-task\\[-2pt]policy};
\node[above=0.1cm of U1, font=\scriptsize, align=center] {Unsafety labelling\\[-2pt]function};

\node[box, below=1.0cm of $(pi1)!0.5!(U1)$] (Dsafe) {$D_{\text{source}}^{\text{safe}}$};
\node[right=0.3cm of Dsafe, font=\scriptsize, align=left] {\textcolor{blue!70!black}{\textbf{Phase 1}} -- Safe behaviour\\[-2pt]demonstration};

\draw[arr] (U1) -- (Dsafe);

\node[box, below=0.9cm of Dsafe, fill=orange!20] (theta1) {$\Theta_{\text{source}}^{\text{safe}}$};
\node[right=0.3cm of theta1, font=\scriptsize, align=left] {\textcolor{blue!70!black}{\textbf{Phase 2}} -- Safe parameter region\\[-2pt](Rashomon set / maximal LID)};

\draw[arr] (Dsafe) -- (theta1);

\node[circlebox, below=1.0cm of theta1] (pi2) {$\pi_{\text{down}}$};
\node[right=0.3cm of pi2, font=\scriptsize, align=left] {\textcolor{blue!70!black}{\textbf{Phase 3}} -- Downstream policy:\\[-2pt]$\bullet$ adapted to new task\\[-2pt]$\bullet$ safe on source task};

\draw[arr] (theta1) -- (pi2);

\draw[arr, rounded corners=5pt]
    (pi1.south) -- (pi1 |- theta1) -- (theta1.west);

\end{tikzpicture}
\caption{Overview of the proposed method. The unsafety labelling function $U_{\mathrm{source}}$ yields safe demonstrations $D_{\mathrm{source}}^{\mathrm{safe}}$ \textbf{(Phase~1)}. A certified safe parameter region $\Theta_{\mathrm{source}}^{\mathrm{safe}}$ is computed around $\pi_{\mathrm{source}}$ \textbf{(Phase~2)}, and the downstream policy $\pi_{\mathrm{down}}$ is adapted within this region \textbf{(Phase~3)}.}
\label{fig:method_diagram}
\end{figure}

\subsection{Phase 1: Safe Behaviour Demonstration}\label{sec:phase1}

In supervised learning, the LID specification is evaluated on a fixed dataset independent of model parameters. In RL, the state distribution shifts with the policy, creating a circular dependency. We resolve this by constructing a \emph{distribution-independent} safety constraint: using $U_{\mathrm{source}}$, we list all safe state--action pairs and form the safe demonstration dataset:
\begin{equation}\label{eq:Dsafe}
    D_{\mathrm{source}}^{\mathrm{safe}} = \bigl\{(s,\, \mathcal{A}^{\mathrm{safe}}(s)) : s \in \mathcal{S}_{\mathrm{sc}} \bigr\}.
\end{equation}
States where all actions are safe are excluded from the safety demonstration dataset due to redundancy. By enforcing safety at every state in $D_{\mathrm{source}}^{\mathrm{safe}}$, we guarantee safety along any trajectory that an updated policy can have in the source task.

\begin{figure}[t]
    \centering
    \includegraphics[width=0.7\linewidth]{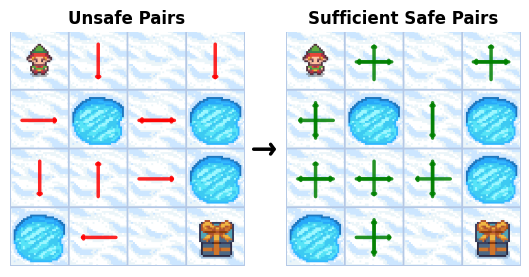}
    \caption{Safe demonstration dataset construction in Frozen Lake environment. We label state-action pairs that lead to ice holes as unsafe. Then, we derive a demonstration of safe state-action pairs, which exclude states without unsafe actions.}
    \label{fig:safe-behaviour-demo}
\end{figure}

\subsection{Phase 2: Certified Safe Region}\label{sec:phase3}

We compute the maximal $\Theta_{\mathrm{source}}^{\mathrm{safe}}$ centred at $\theta_{\mathrm{source}}$: the largest orthotope in parameter space within which policies are safe on $D_{\mathrm{source}}^{\mathrm{safe}}$. Note that if $\pi_{\mathrm{source}}$ does not satisfy the safety specification, the Rashomon set is empty and the algorithm terminates.

\paragraph{Rashomon set computation.}
We define the safe region as a centre-symmetric orthotope $
\{\theta': \theta_{\mathrm{source}}-\alpha \le \theta' \le \theta_{\mathrm{source}}+\alpha\}$ with half-widths $\alpha \in \mathbb{R}_{\geq 0}^p$. We maximise its volume subject to the lower bound of the safety surrogate inside the orthotope:
\begin{equation}\label{eq:max_lid_method}
    \max_{\alpha \geq 0}\; \sum_{i=1}^p \log \alpha_i \quad \text{s.t.}\quad \min_{\alpha \geq 0} \widetilde{\Phi}(\theta_{\mathrm{source}}, \alpha) \geq \delta,
\end{equation}
solved via the primal--dual algorithm \citep{elmeckerplakolm2025provably}. 

\subsubsection{Safety Specification: Hard vs Surrogate}

We distinguish between a \emph{hard} (exact) safety specification, which is exact but non-differentiable, and a \emph{soft} (surrogate) specification, which is differentiable but only provides a one-sided guarantee.


\begin{definition}[Hard safety specification]\label{def:hard_safety_specification}
A policy satisfies the hard safety specification at state $s$ if $\phi^{\mathrm{safe}}(s) = 1$, where:
\begin{equation}\label{eq:true-spec}
  \phi^{\mathrm{safe}}(s)
  :=
  \ind\!\left\{
    \argmax_{a \in \mathcal{A}(s)} z_\mathcal{A}(s)
    \in \mathcal{A}^{\mathrm{safe}}(s)
  \right\},
\end{equation}
assuming no ties at the maximum.
\end{definition}

The hard specification is:
\begin{itemize}
    \item \textbf{sound and complete}: it exactly captures safety,
    \item \textbf{non-differentiable}: due to the $\argmax$ and indicator.
\end{itemize}

We want to verify that a policy is safe in any safety-critical state:
\begin{equation*}
    \phi^{\mathrm{safe}}(s) = 1 \quad \forall s \in \mathcal{S}_{\mathrm{sc}}.
\end{equation*}

To enable optimisation, we introduce a smooth \emph{safety surrogate}:

\begin{definition}[Safety surrogate]\label{def:safety_surrogate}
A safety surrogate is $\widetilde{\Phi}^{\mathrm{safe}}_{\tau}(s)$:
\begin{equation}
\widetilde{\phi}^{\mathrm{safe}}_{\tau}(s)
:=
\sum_{a \in \mathcal{A}^{\mathrm{safe}}(s)} \pi(a|s; \tau),
\end{equation}
where
\begin{equation*}
\pi(a|s; \tau)
=
\dfrac{e^{z_\mathcal{A}(s)/\tau}}{\sum_{a' \in \mathcal{A}(s)} e^{z_{a'}(s)/\tau}},
\quad \tau > 0.
\end{equation*}
\end{definition}

\begin{proposition}[Properties of the surrogate]\label{proposition:surrogate_properties}
Assume no ties at the maximum. Then:
\begin{enumerate}
    \item \textbf{Differentiability:} $\widetilde{\phi}^{\mathrm{safe}}_{\tau}(s)$ is smooth in $z_\mathcal{A}(s)$ for all $\tau > 0$.
    
    \item \textbf{Consistency:}
    \[
    \lim_{\tau \to 0^+}
    \widetilde{\phi}^{\mathrm{safe}}_{\tau}(s)
    =
    \phi^{\mathrm{safe}}(s).
    \]
    
    \item \textbf{Soundness (one-sided guarantee):} for all $\tau > 0$,
    \[
    \widetilde{\phi}^{\mathrm{safe}}_{\tau}(s)
    >
    \frac{|\mathcal{A}^{\mathrm{safe}}(s)|}{1 + |\mathcal{A}^{\mathrm{safe}}(s)|}
    \;\;\Rightarrow\;\;
    \phi^{\mathrm{safe}}(s) = 1.
    \]
    
    \item \textbf{Non-completeness (for finite $\tau$):}
    \[
    \widetilde{\phi}^{\mathrm{safe}}_{\tau}(s)
    \leq
    \frac{|\mathcal{A}^{\mathrm{safe}}(s)|}{1 + |\mathcal{A}^{\mathrm{safe}}(s)|}
    \;\;\not\Rightarrow\;\;
    \phi^{\mathrm{safe}}(s) = 0.
    \]
\end{enumerate}
\end{proposition}

The surrogate provides a \emph{sufficient but not necessary} condition for safety:
high surrogate values certify safety, but low values are inconclusive. Therefore, we impose the following \emph{sound} constraint:
\begin{equation}
\widetilde{\phi}^{\mathrm{safe}}_{\tau}(s)
>
\frac{|\mathcal{A}^{\mathrm{safe}}(s)|}{1 + |\mathcal{A}^{\mathrm{safe}}(s)|}.
\end{equation}

\subsubsection{Global Safety Specification}

We now lift the specification from states to policies.

\begin{definition}[Critical State Safety Rate]\label{def:critical_state_safety_rate}
Define a critical state safety rate as a proportion of safety-critical states in which the policy is safe:
\begin{equation*}
\Phi^{\mathrm{safe}}_{\mathrm{sc}}(\pi)
=
\frac{1}{|\mathcal{S}_{\mathrm{sc}}|}
\sum_{s \in \mathcal{S}_{\mathrm{sc}}}
\phi^{\mathrm{safe}}(s).
\end{equation*}
\end{definition}

\begin{definition}[Trajectory Safety Rate]\label{def:trajectory_safety_rate}
Define a trajectory safety rate as an expected proportion of episodes in which the policy trajectory does not experience any unsafe state-action pairs:
\begin{equation*}
\Phi^{\mathrm{safe}}_{\mathrm{traj.}}(\pi)
=
\mathbb{E}_{s \sim d^\pi}
\left[
\phi^{\mathrm{safe}}(s)
\right],
\end{equation*}
where $d^\pi(s)$ is the state visitation distribution induced by $\pi$.
\end{definition}

\subsubsection{Global Safety Surrogate}
During optimisation, we replace the hard global specification with a smooth lower bound.

\begin{definition}[Critical State Safety Surrogate]\label{def:critical_state_safety_surrogate}
The critical state safety surrogate is defined as the minimum per-state safety surrogate in safety-critical states:
\begin{equation*}
\widetilde{\Phi}^{\mathrm{safe}}_{\tau, \mathrm{sc}}(\pi)
=
\min_{s \in \mathcal{S}_{\mathrm{sc}}}
\widetilde{\phi}^{\mathrm{safe}}_{\tau}(s).
\end{equation*}
\end{definition}

\begin{definition}[Sound global bound]\label{def:sound_global_bound}
Let
\[
M := \max_{s \in \mathcal{S}_{\mathrm{sc}}}
|\mathcal{A}^{\mathrm{safe}}(s)|.
\]
Then
\begin{equation*}
\widetilde{\Phi}^{\mathrm{safe}}_{\tau, \mathrm{sc}}(\pi)
>
\frac{M}{1 + M}
\;\;\Rightarrow\;\;
\Phi^{\mathrm{safe}}_{\mathrm{sc}}(\pi) = 1.
\end{equation*}
\end{definition}

\paragraph{Key implication.}
The surrogate enables gradient-based optimisation while preserving a \emph{global safety certificate}: if the global surrogate constraint is satisfied with $\delta^{\star} = \frac{M}{1+M}$, the global hard safety specification constraint $\Phi_{sc}(\pi)=1$ is satisfied as well.

\subsection{Phase 3: Safe Downstream Adaptation}\label{sec:phase4}

We adapt to the downstream task by solving $\max_{\theta \in \Theta_{\mathrm{source}}^{\mathrm{safe}}} J_{\mathrm{down}}(\theta)$ via a gradient-based policy optimisation method (e.g. PPO) with an additional projection step. After each gradient step $\hat{\theta} = \theta + \eta\, g$, we project back via element-wise clipping:
\begin{equation}\label{eq:project}
    \theta_{\mathrm{down}} \leftarrow \mathrm{clip}\!\big(\hat{\theta},\;\; \theta_{\mathrm{source}} - \alpha^*,\;\; \theta_{\mathrm{source}} + \alpha^*\big).
\end{equation}

\subsection{Assumptions and Theoretical Guarantees}\label{sec:assumptions}

We now state the assumptions under which our safety guarantees hold.

\begin{assumption}[Finite discrete state--action space]\label{ass:finite}
The state space $\mathcal{S}$ and action space $\mathcal{A}$ are discrete and finite. This enables using unsafety labelling function
$U_{\mathrm{source}}$ to generate the sufficient safety demonstration dataset $D_{\mathrm{source}}^{\mathrm{safe}}$ for the source task.
\end{assumption}

\begin{assumption}[Existence of safe actions]\label{ass:safe-actions}
For every safety-critical state,
there exists at least one safe action:
$\mathcal{A}^{\mathrm{safe}}(s) \neq \emptyset \; \forall s \in \mathcal{S}_{\mathrm{sc}}$. Otherwise, the method cannot guarantee safe behaviour in any state.
\end{assumption}

\begin{assumption}[Safe source policy]\label{ass:safe-source}
The source-task policy $\pi_{\theta_{\mathrm{source}}}$ satisfies
the sound safety surrogate constraint at the required threshold:
$$\widetilde{\Phi}^{\mathrm{safe}}_{\tau, \mathrm{sc}}(\pi)
>
\frac{M}{1 + M}.$$
\end{assumption}
This enables building a Rashomon set, which is convex by design.

\begin{assumption}[Greedy action selection]\label{ass:greedy}
At deployment, the agent selects actions greedily, i.e. 
$a^{\star} = \argmax_{a' \in \mathcal{A}(s)} \pi_{\theta}(a' | s)$. This is a standard assumption about test-time policy deployment which also avoids formal analysis of stochastic sampling from $\pi_{\theta}(\cdot \mid s)$.
\end{assumption}

\noindent Assumption~\ref{ass:finite} restricts the formulation to discretised state-action spaces, ensuring that the safe demonstration dataset $D_{\mathrm{source}}^{\mathrm{safe}}$ can be constructed exhaustively. We note that this does not require manual enumeration: in practice, the unsafety labelling function is specified as a computable predicate grounded in domain knowledge (e.g., ``if adjacent to a hole, moving toward it is unsafe''), and the dataset is constructed programmatically. Assumption~\ref{ass:safe-actions} excludes states from which failure is unavoidable under any policy. Assumption~\ref{ass:safe-source} is a necessary precondition: no certified convex region can exist around a policy that is itself unsafe. Algorithm~\ref{alg:safe_continual_rl} checks this explicitly and returns $\bot$ if the condition is violated. Assumption~\ref{ass:greedy} avoids dealing with stochastic action sampling when an unsafe action can be drawn even if the greedy action is safe.

\begin{theorem}[Provably safe policy updates]
\label{thm:main_safeadapt}
Let Assumptions~1--4 hold, and let
\[
\Theta^{\mathrm{safe}}_{\mathrm{source}}
=
\{\theta' : \theta_{\mathrm{source}}-\alpha^\star \le \theta' \le \theta_{\mathrm{source}}+\alpha^\star\}
\]
be the certified orthotope returned by Phase~2 of Algorithm~1.
Let
\[
M := \max_{s \in \mathcal{S}_{\mathrm{sc}}} |\mathcal{A}^{\mathrm{safe}}(s)|,
\qquad
\delta^\star := \frac{M}{1+M}.
\]
Assume that the verification procedure used in Phase~2 provides a sound lower bound for the safety surrogate $\widetilde{\Phi}^{\mathrm{safe}}_{\tau,\mathrm{sc}}$ such that
\[
\min_{\theta' \in \Theta^{\mathrm{safe}}_{\mathrm{source}}}\widetilde{\Phi}^{\mathrm{safe}}_{\tau,\mathrm{sc}}(\pi_{\theta'}) > \delta^\star .
\]
Then, \(\forall \; \theta' \in \Theta^{\mathrm{safe}}_{\mathrm{source}}\),
\[
\Phi^{\mathrm{safe}}_{\mathrm{sc}}(\pi_{\theta'}) = 1.
\]
In particular, if Phase~3 updates the policy by projected gradient descent onto
\(\Theta^{\mathrm{safe}}_{\mathrm{source}}\), then in every iteration \(\theta_t\) satisfies
\[
\Phi^{\mathrm{safe}}_{\mathrm{sc}}(\pi_{\theta_t}) = 1.
\]
\end{theorem}

\begin{proof}
Take any \(\theta' \in \Theta^{\mathrm{safe}}_{\mathrm{source}}\) and any \(s \in \mathcal{S}_{\mathrm{sc}}\).
By Definition~4.5,
\[
\widetilde{\Phi}^{\mathrm{safe}}_{\tau,\mathrm{sc}}(\pi_{\theta'})
=
\min_{s' \in \mathcal{S}_{\mathrm{sc}}} \widetilde{\phi}^{\mathrm{safe}}_{\tau}(s';\pi_{\theta'}),
\]
hence
\[
\widetilde{\phi}^{\mathrm{safe}}_{\tau}(s;\pi_{\theta'})
\ge
\widetilde{\Phi}^{\mathrm{safe}}_{\tau,\mathrm{sc}}(\pi_{\theta'}).
\]
Because the verification lower bound is sound over \(\Theta^{\mathrm{safe}}_{\mathrm{source}}\),
\[
\widetilde{\Phi}^{\mathrm{safe}}_{\tau,\mathrm{sc}}(\pi_{\theta'})
\ge
\min_{\theta \in \Theta^{\mathrm{safe}}_{\mathrm{source}}}{\widetilde{\Phi}}^{\mathrm{safe}}_{\tau,\mathrm{sc}}(\pi_{\theta})
>
\delta^\star
=
\frac{M}{1+M}.
\]
Since \(|\mathcal{A}^{\mathrm{safe}}(s)| < M\) and the map \(x \mapsto x/(1+x)\) is increasing on
\([0,\infty)\),
\[
\frac{M}{1+M}
>
\frac{|\mathcal{A}^{\mathrm{safe}}(s)|}{1+|\mathcal{A}^{\mathrm{safe}}(s)|}.
\]
Therefore
\[
\widetilde{\phi}^{\mathrm{safe}}_{\tau}(s;\pi_{\theta'})
>
\frac{|\mathcal{A}^{\mathrm{safe}}(s)|}{1+|\mathcal{A}^{\mathrm{safe}}(s)|}.
\]
By Proposition~\ref{proposition:surrogate_properties}(3), this implies
\[
\phi^{\mathrm{safe}}(s;\pi_{\theta'}) = 1.
\]
Since this holds for every \(s \in \mathcal{S}_{\mathrm{sc}}\), Definition~4.3 yields
\[
\Phi^{\mathrm{safe}}_{\mathrm{sc}}(\pi_{\theta'}) = 1.
\]
Finally, Phase~3 projects each update back into \(\Theta^{\mathrm{safe}}_{\mathrm{source}}\), so every
iterate \(\theta_t\) remains in \(\Theta^{\mathrm{safe}}_{\mathrm{source}}\); applying the same argument to
each \(\theta_t\) proves the second claim.
\end{proof}

\begin{corollary}[Per-state certified safety]\label{cor:per_state_certified_safety}
Under the assumptions of Theorem~\ref{thm:main_safeadapt}, for every
\(\theta' \in \Theta^{\mathrm{safe}}_{\mathrm{source}}\) and every \(s \in \mathcal{S}_{\mathrm{sc}}\),
\[
\arg\max_{a \in \mathcal{A}(s)} \pi_{\theta'}(a \mid s) \in \mathcal{A}^{\mathrm{safe}}(s).
\]
\end{corollary}

\begin{corollary}[Safety preservation throughout adaptation]\label{cor:safety_preservation_throughout_adaptation}
Let \((\theta_t)_{t \ge 0}\) be the sequence of policy parameters produced by
Phase~3. Then \(\theta_t \in \Theta^{\mathrm{safe}}_{\mathrm{source}}\) for all \(t\), and therefore
\[
\Phi^{\mathrm{safe}}_{\mathrm{sc}}(\pi_{\theta_t}) = 1
\qquad \forall t \ge 0.
\]
\end{corollary}

\begin{corollary}[Distribution-independent source-task safety]\label{cor:distribution_independent_source_task_safety}
Under the assumptions of Theorem~\ref{thm:main_safeadapt}, greedy execution of
\(\pi_{\theta_t}\) on the source task never selects an unsafe action at any iteration
\(t\). Consequently, for any initial state distribution \(\mu_0\), the source-task
occupancy measure of \(\pi_{\theta_t}\) satisfies
\[
\mathbb{E}_{(s,a)\sim d^{\pi_{\theta_t}}}\!\left[U_{\mathrm{source}}(s,a)\right] = 0.
\]
\end{corollary}

Note that Theorem~\ref{thm:main_safeadapt} guarantees preservation of source-task safety, but it does not guarantee that either source or downstream goal-reaching policy exists inside the certified region. Whether such a policy exists depends on the overlap between the certified safe set and high-performing source-task and downstream-task policies.

\section{Experiments}\label{sec:experiments}

We consider continual adaptation from a source task (Task~1) to a downstream task (Task~2) in discrete-state, discrete-action environments with known unsafe state-action pairs. In all experiments, \(\pi_{\text{source}}\) is safe on the source-task safety dataset before adaptation. In our finite-state settings, this corresponds to full safety coverage of the source safety-critical states. We compare four policies:
\begin{itemize}
    \item \textbf{Source}: train the policy on the source task and without updating parameters to the downstream task.
    \item \textbf{UnsafeAdapt}: unconstrained PPO fine-tuning on Task~2.
    \item \textbf{EWC}: PPO fine-tuning on Task~2 with EWC regularisation.
    \item \textbf{\textsc{SafeAdapt} (ours)}: PPO fine-tuning on Task~2 with per-update projection onto the Rashomon set, which is a parameter set certified to be safe in the source task.
\end{itemize}

\paragraph{Policy training.}
In each experiment run, policies are trained with PPO \citep{schulman2017ppo} using the same MLP actor and critic with softmax over action logits. This ensures that the expressiveness of the neural policy is the same across methods.
We run 10 seeds and report mean \(\pm\) standard deviation. Detailed settings are specified in Tables~\ref{tab:stage1_source}-\ref{tab:stage3_ewc}.

\paragraph{Metrics.} For each environment and policy \(\pi\), we calculate the following metrics for comparison:
\begin{itemize}
    \item \textbf{Critical-state safety rate} $\Phi^{\mathrm{safe}}_{\mathrm{sc}}(\pi)$: fraction of safety-critical states in which the policy is safe.
    \item \textbf{Trajectory safety rate} $\Phi^{\mathrm{safe}}_{\mathrm{traj.}}(\pi)$: fraction of episodes in which the policy trajectory is safe.
    \item \textbf{Total reward}: Episodic return of the policy.
    \item \textbf{Success rate}: environment-defined task completion rate.
\end{itemize}
\looseness=-1
In each experiment run, our method \textsc{SafeAdapt} is the only method that provides a formal guarantee that $\Phi^{\mathrm{safe}}_{\mathrm{sc}}(\pi) = 1$ in a source task.

\subsection{Environments}

\subsubsection{Frozen Lake}\label{sec:exp_frozenlake}
Frozen Lake is a grid-world Gymnasium environment \citep{towers2025gymnasium} where the agent must reach the goal while avoiding holes in a frozen lake. State is represented with a one-hot encoded location and task indicator, and there are 4 actions available (go up, down, left, or right). Reward is \(+1\) at goal and \(0\) otherwise. Task~1 and Task~2 differ by hole placement. Falling into a hole is unsafe, and an episode is terminated upon this event. 

\looseness=-1
\subsubsection{Poisoned Apple}\label{sec:exp_poisoned_apple}
Poisoned Apple is a custom grid world in which an agent's goal is to collect safe apples while avoiding poisoned ones. Collecting a safe apple yields \(+1\), collecting a poisoned apple yields \(-1\) and is unsafe. A state is represented as flattened grid with locations of safe and poisoned apples indicated therein. There are 4 available actions: up, down, left and right. Task~1 and Task~2 differ in safe/poisoned apple placement. A trajectory is safe if no poisoned apple is collected. Once all safe apples are collected, the environment terminates.\\ 

\looseness=-1
Figure~\ref{fig:envs_demonstration} demonstrates the environments and their corresponding source and downstream tasks. Table~\ref{tab:frozen_poison_comparison} showcases that we have a setup which tests \textsc{SafeAdapt} across heterogeneous environments. Frozen Lake is a Task-Incremental Learning (TIL) setting, since the task ID is included in the state and the policy can condition explicitly on task context. In contrast, Poisoned Apple is a Domain-Incremental Learning (DIL) setting, where adaptation must occur without explicit task identity, making it a harder regime than TIL. We use small discrete environments because they permit a straightforward construction of the safety-critical-state dataset and end-to-end validation of the deterministic certificate under the assumptions of Theorem 1.
\begin{figure}[ht]
    \centering
    \includegraphics[width=0.7\linewidth]{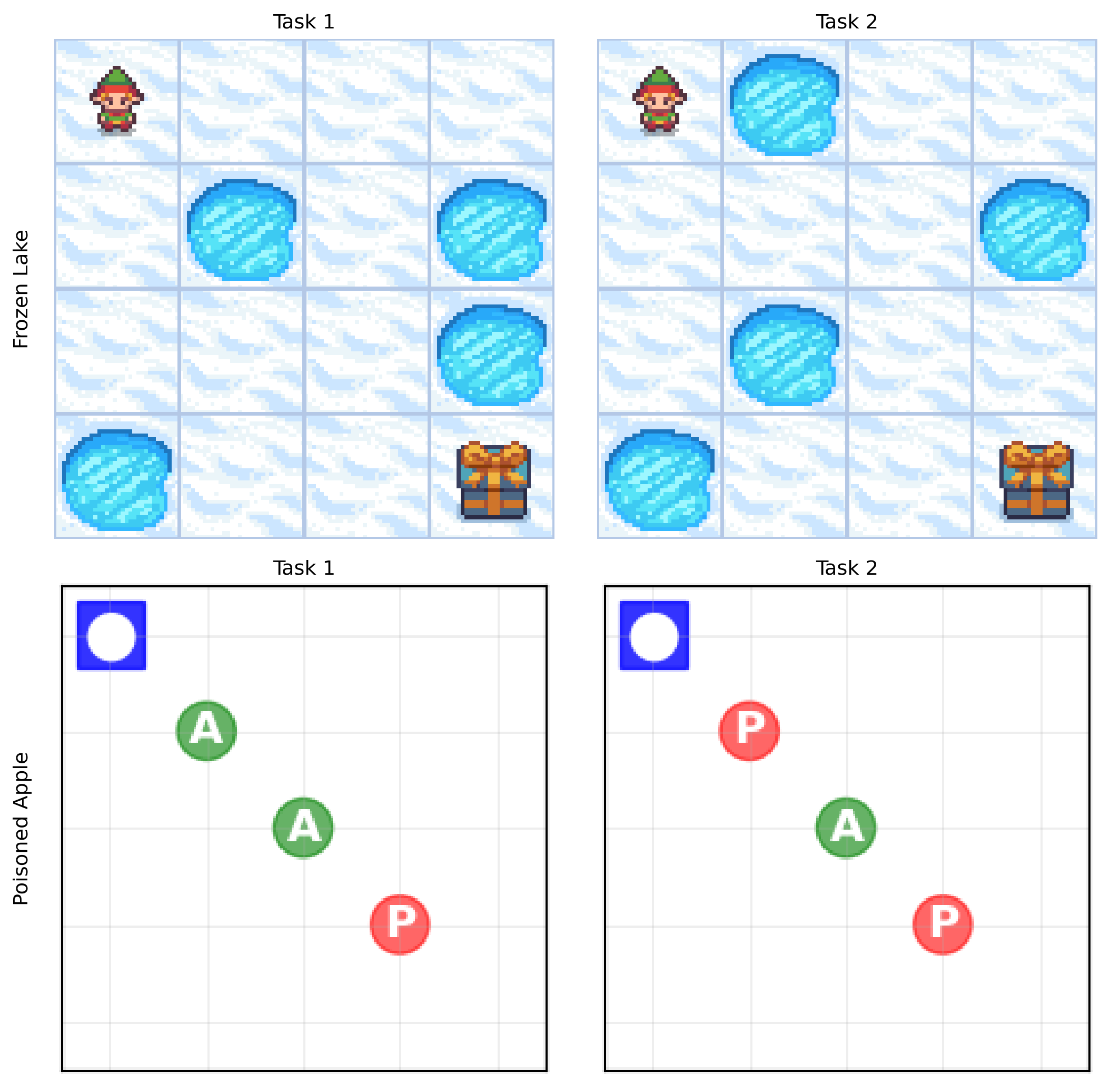}
    \caption{Examples of experiment environments: Frozen Lake \texttt{standard\_4x4} (first row) and Poisoned Apple  \texttt{simple\_5x5} (second row).}
    \label{fig:envs_demonstration}
\end{figure}

\begin{table}[ht]
\centering
\small
\caption{Frozen Lake vs Poisoned Apple experiment structure.}
\label{tab:frozen_poison_comparison}
\begin{tabular}{lcc}
\toprule
Setting & Frozen Lake & Poisoned Apple \\
\midrule
Task ID in state representation & Yes & No \\
Termination when unsafe & Yes & No \\
Single destination & Yes & No \\
\bottomrule
\end{tabular}
\end{table}

\subsection{Results}

\subsubsection{Safety retention in the source task.} Table~\ref{tab:safety_retention} demonstrates how our method (\textsc{SafeAdapt}) retains agent's safety in the source task after the update to a new task. Out of all adaptation methods, only \textsc{SafeAdapt} has the average critical state safety rate $\Phi^{\mathrm{safe}}_{\mathrm{sc}}(\pi)$ of 1. That is, UnsafeAdapt and EWC methods \emph{forget} how to be safe in some safety-critical states. The average trajectory safety rate $\Phi^{\mathrm{safe}}_{\mathrm{traj.}}(\pi)$ of \textsc{SafeAdapt} agent is 1, which follows from \textsc{SafeAdapt}'s safety in all safety-critical states. UnsafeAdapt and EWC agents have unsafe policy trajectories in the source task in Frozen Lake environment, while their trajectories are safe in Poisoned Apple. Finally, the Total Reward column demonstrates degradation of agent's performance in the source task across all adaptation methods. However, due to the nature of Frozen Lake and Poisoned Apple, safety retention allows the \textsc{SafeAdapt} to maintain competitive performance in the source task compared to other adaptation methods.

\begin{table*}[ht]
\centering
\setlength{\tabcolsep}{3pt}
\caption{
Source Task safety and performance metrics in Frozen Lake and Poisoned Apple (mean $\pm$ standard deviation over 10 seeds). We report (i) critical state safety rate $\Phi^{\mathrm{safe}}_{\mathrm{sc}}(\pi)$, measuring the fraction of safety-critical states in which the policy selects only safe actions, (ii) trajectory-level safety $\Phi^{\mathrm{safe}}_{\mathrm{traj.}}(\pi)$, measuring the probability of executing a fully safe trajectory, and (iii) total reward. 
Standard continual learning baselines (UnsafeAdapt and EWC) may violate safety constraints, as reflected by degraded $\Phi^{\mathrm{safe}}_{\mathrm{sc}}$ in both Frozen Lake and Poisoned Apple and degraded $\Phi^{\mathrm{safe}}_{\mathrm{traj.}}$ in Frozen Lake. In contrast, \textsc{SafeAdapt} (ours) consistently retains perfect safety ($\Phi^{\mathrm{safe}}_{\mathrm{sc}} = \Phi^{\mathrm{safe}}_{\mathrm{traj.}} = 1$) across all environments while maintaining competitive total reward in the source task. 
}
\label{tab:safety_retention}
\begin{tabularx}{\textwidth}{
llc 
>{\centering\arraybackslash}X 
>{\centering\arraybackslash}X 
>{\centering\arraybackslash}X
}
\toprule
Environment & Policy & Provably Safe? & $\Phi^{\mathrm{safe}}_{\mathrm{sc}}(\pi)$ & $\Phi^{\mathrm{safe}}_{\mathrm{traj.}}(\pi)$ & Total Reward \\
\midrule

\multirow{4}{*}{Frozen Lake (\texttt{standard\_4x4})} & Source & \checkmark & $\mathbf{1.00 \pm 0.00}$ & $\mathbf{1.00 \pm 0.00}$ & $\mathbf{1.00 \pm 0.00}$ \\
 & UnsafeAdapt & $\times$ & $0.88 \pm 0.00$ & $0.10 \pm 0.30$ & $0.00 \pm 0.00$ \\
 & EWC & $\times$ & $0.94 \pm 0.06$ & $0.60 \pm 0.49$ & $0.30 \pm 0.46$ \\
 & \textsc{SafeAdapt} (ours) & \checkmark & $\mathbf{1.00 \pm 0.00}$ & $\mathbf{1.00 \pm 0.00}$ & $0.90 \pm 0.30$ \\
\midrule

\multirow{4}{*}{Poisoned Apple (\texttt{simple\_5x5})} 
 & Source & \checkmark & $\mathbf{1.00 \pm 0.00}$ & $\mathbf{1.00 \pm 0.00}$ & $\mathbf{1.96 \pm 0.00}$ \\
 & UnsafeAdapt & $\times$ & $0.93 \pm 0.11$ & $\mathbf{1.00 \pm 0.00}$ & $0.91 \pm 0.00$ \\
 & EWC & $\times$ & $0.93 \pm 0.11$ & $\mathbf{1.00 \pm 0.00}$ & $1.02 \pm 0.32$ \\
 & \textsc{SafeAdapt} (ours) & \checkmark & $\mathbf{1.00 \pm 0.00}$ & $\mathbf{1.00 \pm 0.00}$ & $0.91 \pm 0.00$ \\
 \midrule
 
\bottomrule
\end{tabularx}
\end{table*}
\looseness=-1
\subsubsection{Adaptation to downstream task.} Table~\ref{tab:plasticity_analysis} shows that \textsc{SafeAdapt} retains substantial downstream plasticity despite the hard source-task safety constraint. In Frozen Lake, all three adaptation methods achieve optimal downstream performance, so the relevant distinction is not Task-2 total reward but whether that reward is obtained without sacrificing source-task safety. In this respect, \textsc{SafeAdapt} is strictly preferable, as it matches the best downstream performance while being the only method with a source-task safety certificate. In Poisoned Apple, \textsc{SafeAdapt} again matches the best-performing baseline, achieving the same average total reward as UnsafeAdapt (0.96) and outperforming EWC (0.86). Overall, these results show that constraining updates to the certified Rashomon set can allow for effective downstream adaptation and can preserve source-task safety without sacrificing downstream performance.

\begin{table*}[ht]
\centering
\caption{Downstream Task performance metrics in Frozen Lake and Poisoned Apple (mean $\pm$ standard deviation over 10 seeds).}
\label{tab:plasticity_analysis}
\begin{tabularx}{\textwidth}{
l l 
>{\centering\arraybackslash}X 
>{\centering\arraybackslash}X
}
\toprule
Environment & Policy & Total Reward & Success Rate \\
\midrule

\multirow{4}{*}{Frozen Lake (\texttt{standard\_4x4})} & Source & $0.00 \pm 0.00$ & $0.00 \pm 0.00$ \\
 & UnsafeAdapt & $\mathbf{1.00 \pm 0.00}$ & $\mathbf{1.00 \pm 0.00}$ \\
 & EWC & $\mathbf{1.00 \pm 0.00}$ & $\mathbf{1.00 \pm 0.00}$ \\
 & \textsc{SafeAdapt} (ours) & $\mathbf{1.00 \pm 0.00}$ & $\mathbf{1.00 \pm 0.00}$ \\
\midrule

\multirow{4}{*}{Poisoned Apple (\texttt{simple\_5x5})} 
& Source & $-0.04 \pm 0.00$ & $\mathbf{1.00 \pm 0.00}$ \\
& UnsafeAdapt & $\mathbf{0.96 \pm 0.00}$ & $\mathbf{1.00 \pm 0.00}$ \\
& EWC & $0.86 \pm 0.30$ & $\mathbf{1.00 \pm 0.00}$ \\
& \textsc{SafeAdapt} (ours) & $\mathbf{0.96 \pm 0.00}$ & $\mathbf{1.00 \pm 0.00}$ \\
 
\bottomrule
\end{tabularx}
\end{table*}
\looseness=-1
\subsubsection{Scalability analysis}
To study scalability, we evaluate how \textsc{SafeAdapt} behaves as the Frozen Lake layout size increases. Figure~\ref{fig:scalability_analysis} shows a stability--plasticity trade-off across scales. The source policy has perfect Task-1 safety and total reward, but fails completely on Task~2. UnsafeAdapt and EWC exhibit strong downstream performance, yet incur substantial degradation on the source task. Namely, UnsafeAdapt exhibits a large drop in critical-state safety and zero trajectory safety rate, while EWC retains partial source-task safety and near-zero Task-1 total reward on the larger layouts. In contrast, \textsc{SafeAdapt} maintains perfect source-task safety across all sizes and remains competitive on the downstream task, with only a modest drop at 8×8. To keep the comparison controlled, we use the same policy architecture and training configuration across all Frozen Lake layouts. Larger layouts may benefit from increased model capacity and optimised training settings. Exploring scaling of the environment, network architecture, and training procedure is a potential direction for future work.

\begin{figure}[ht]
    \centering
    \includegraphics[width=1\linewidth]{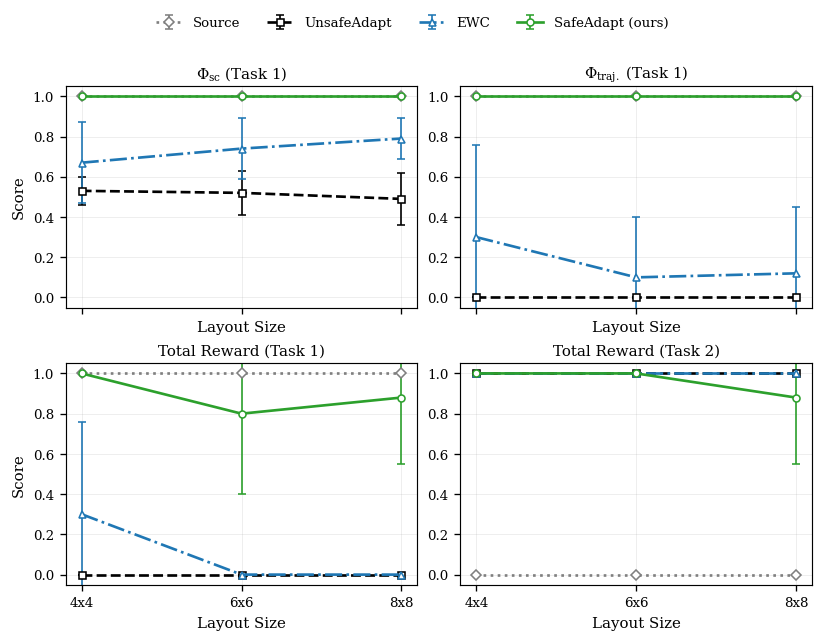}
    \caption{Scalability analysis across diagonal Frozen Lake layouts.
    }
    \label{fig:scalability_analysis}
\end{figure}

Appendix~\ref{appendix:Rashomon_set_vis} visualises the Rashomon set for Frozen Lake. We leave analysis of certified-region size and its dependence on neural network architecture to future work.

\section{Conclusion}\label{sec:conclusion}

We introduced \textsc{SafeAdapt}, an approach to provably safe policy updates in continual deep reinforcement learning. The key idea is to compute a certified Rashomon set in policy parameter space around a safe source-task policy and to constrain downstream adaptation to remain within this set via projected gradient descent. Our method yields an \emph{a priori} guarantee that source-task safety is preserved throughout downstream adaptation, rather than only being checked \emph{a posteriori} or enforced by a runtime intervention mechanism.

Empirically, we showed across grid-world environments Frozen Lake and Poisoned Apple that \textsc{SafeAdapt} is the only adaptation method in our study that preserves perfect source-task safety after updates to a downstream task, while exhibiting competitive downstream performance. In contrast, unconstrained adaptation (UnsafeAdapt) and EWC can both exhibit catastrophic forgetting of safe behaviour in a source task while adapting to a new task. The scalability study further suggests that constrained adaptation within the safe region may exhibit trade-off between safety retention and downstream plasticity as environment size increases.

The current formulation is limited to finite discrete state--action spaces as it relies on exhaustive coverage of safety-critical states to obtain safety guarantees. In addition, the certified parameter region may become conservative when using IBP, especially for larger networks or more complex environments. Extending the framework to infinite and continuous state-action spaces and understanding how certified regions behave in larger environments and across longer task sequences are important directions for future work. Another interesting idea is to derive the unsafety labelling function from a synthesised shield, which could extend our framework to more expressive multi-step safety specifications. Finally, probabilistic verification is a complementary direction for more complex settings where deterministic parameter-space certification becomes too conservative or intractable.

Overall, our results show that parameter-space certification is a viable route to preventing safety forgetting in continual RL and provide, to our knowledge, the first parameter-space certificate for preserving source-task safety during downstream adaptation.

\begin{acks}
This work was supported by the UKRI Centre for Doctoral Training in Safe and Trusted AI [EP/S0233356/1] and the EPSRC grant number EP/X015823/1, "An abstraction-based Technique for Safe Reinforcement Learning”.

\end{acks}

\newpage

   \bibliographystyle{ACM-Reference-Format}
   \bibliography{paper/references/references}
%







\newpage

\appendix

\section{Methodology details}
Here we provide some additional details on our method \textsc{SafeAdapt}.
\subsection{Safety surrogate sound constraint}

Consider a multi-label classification setting with $K$ actions, and without loss of generality, assume the number of actions is the same for each state, i.e. $K = |\mathcal{A}| = |\mathcal{A}(s)|$ $\forall s \in \mathcal{S}$. Also, without loss of generality, assume in each state $M$ actions are safe, i.e. $|\mathcal{A}^{\mathrm{safe}}(s)| = M$ $\forall s \in \mathcal{S}$. Let $\pi(a \mid s)$ denote the policy's probability of selecting action $a$ in state $s$.

The \emph{hard safety specification} is defined as:
\begin{equation*}
    \phisafe(s) = \ind\!\left\{ \argmax_{a}\, \pi(a \mid s) \in \Asafe(s) \right\},
\end{equation*}
which equals 1 if and only if the greedy action under the policy is safe. $\argmax$ and indicator function make $\phisafe$ non-differentiable. To enable gradient-based optimisation, we introduce a safety surrogate, which is defined as:
\begin{equation*}
    \tilde{\phi}^{\mathrm{safe}}(s) = \sum_{a \in \Asafe(s)} \pi(a \mid s),
\end{equation*}
which represents the total softmax probability mass assigned to safe actions.

\subsection{Proof for Proposition~\ref{proposition:surrogate_properties}(3)}
\begin{proof}[Proof of Proposition~\ref{proposition:surrogate_properties}(3)]
    
    From the inequality
    $$
    \sum_{a\in \mathcal{A}^{\mathrm{safe}}(s)} \pi(a| s) >\dfrac{|\mathcal{A}^{\mathrm{safe}}(s)|}{1 + |\mathcal{A}^{\mathrm{safe}}(s)|},
    $$
    we get:
    $$
    \dfrac{\sum_{a\in \mathcal{A}^{\mathrm{safe}}(s)} \pi(a| s)}{|\mathcal{A}^{\mathrm{safe}}(s)|} >\dfrac{1}{1 + |\mathcal{A}^{\mathrm{safe}}(s)|}.
    $$
    Since the average safe action probability, exceeds $\frac{1}{1 + |\mathcal{A}^{\mathrm{safe}}(s)|}$, the maximum safe action probability must exceed $\frac{1}{1 + |\mathcal{A}^{\mathrm{safe}}(s)|}$ as well:
    \begin{equation}\label{eq:max_safe_action_prob}
        \max_{a \in \mathcal{A}^{\mathrm{safe}}(s)}\pi(a|s) > \dfrac{1}{1 + |\mathcal{A}^{\mathrm{safe}}(s)|}
    \end{equation}

    Since the sum of safe action probabilities exceeds $\frac{|\mathcal{A}^{\mathrm{safe}}(s)|}{1 + |\mathcal{A}^{\mathrm{safe}}(s)|}$, the sum of unsafe action probabilities cannot exceed $\frac{1}{1 + |\mathcal{A}^{\mathrm{safe}}(s)|}$:

    $$\sum_{a\notin \mathcal{A}^{\mathrm{safe}}(s)} \pi(a| s) < \dfrac{1}{1 + |\mathcal{A}^{\mathrm{safe}}(s)|}$$ 

    The upper bound for the sum of non-negative values is also the upper bound for the maximum term in the sum:
    \begin{equation}\label{eq:max_unsafe_action_prob}
        \max_{a \notin \mathcal{A}^{\mathrm{safe}}(s)} \pi(a | s) < \dfrac{1}{1 + |\mathcal{A}^{\mathrm{safe}}(s)|}
    \end{equation}
    
    From Eqs.~\ref{eq:max_safe_action_prob} and \ref{eq:max_unsafe_action_prob}, we get:

    \begin{equation*}
         \max_{a \in \mathcal{A}^{\mathrm{safe}}(s)}\pi(a|s) > \max_{a \notin \mathcal{A}^{\mathrm{safe}}(s)} \pi(a | s)
    \end{equation*}
\end{proof}

As the number of safe actions per state goes to infinity ($|\mathcal{A}^{\mathrm{safe}}(s)| \rightarrow \infty$), the surrogate sound constraint becomes very conservative since the threshold approaches 1.

\subsection{Relationship Between $\tilde{\phi}^{\mathrm{safe}}$ and $\phisafe$}

We proved that if the safety surrogate exceeds the threshold $\delta^{\star} = \frac{|\mathcal{A}^{\mathrm{safe}}(s)|}{1 +| \mathcal{A}^{\mathrm{safe}}(s)|}$, the policy is guaranteed to be safe. Note that this condition is sufficient but not necessary.

\begin{proof}[Proof that $\delta^{\star}$ implies a non-necessary constraint.]
Let us provide an example. Consider a state in which one action is safe and 4 actions are unsafe: $|\mathcal{A}(s)|=5$, $|\mathcal{A}^{\mathrm{safe}}(s)|=1$. Then $\delta^{\star} = \frac{1}{2}$. However, a policy with the following action probability distribution is safe:

\begin{equation}
\pi(a|s)=
\begin{cases}
    0.25,  & a \in \mathcal{A}^{\mathrm{safe}}(s), \\
    0.1875 & \text{ otherwise.}
\end{cases}
\end{equation}

Even though $\max_{a \in \mathcal{A}^{\mathrm{safe}}(s)}\pi(a|s) < \delta^{\star}$ ($0.25 < 0.5$), it holds that $\max_{a \in \mathcal{A}^{\mathrm{safe}}(s)}\pi(a|s) > \max_{a \notin \mathcal{A}^{\mathrm{safe}}(s)} \pi(a | s)$ ($0.25 > 0.1875$). This completes the proof that the constraint $\max_{a \in \mathcal{A}^{\mathrm{safe}}(s)}\pi(a|s) < \delta^{\star}$ is not necessary for safety guarantee.
\end{proof}

\subsubsection{Sufficient condition for unsafety} In order to prove that a policy is \emph{unsafe}, it is sufficient to show that the following safety surrogate inequality holds:
\begin{equation*}
    \tilde{\phi}^{\mathrm{safe}}(s) < \dfrac{1}{1 + K - M},
\end{equation*}
where $K = |\mathcal{A}(s)|$ and $M = |\mathcal{A}^{\mathrm{safe}}(s)|$.
\begin{proof}
    Using the definition of the safety surrogate, we write:
    \begin{equation*}
        \sum_{a \in \mathcal{A}^{\mathrm{safe}}(s)}\pi(a|s) < \dfrac{1}{1 + K - M}
    \end{equation*}

    \begin{equation*}
        \sum_{a \notin \mathcal{A}^{\mathrm{safe}}(s)}\pi(a|s) = 1 - \sum_{a \in \mathcal{A}^{\mathrm{safe}}(s)}\pi(a|s) > 1 - \dfrac{1}{1 + K - M}
    \end{equation*}

    \begin{equation*}
        \sum_{a \notin \mathcal{A}^{\mathrm{safe}}(s)}\pi(a|s) > \dfrac{K-M}{1 + K - M}
    \end{equation*}

    From this it follows that:
    \begin{itemize}
        \item $\max_{a \notin \mathcal{A}^{\mathrm{safe}}(s)}\pi(a|s) > \dfrac{1}{1 + K - M}$ 
        \item $\max_{a \in \mathcal{A}^{\mathrm{safe}}(s)} \pi(a | s) < \dfrac{1}{1 + K - M}$
    \end{itemize}   
    Therefore:
    \begin{equation*}
        \max_{a \notin \mathcal{A}^{\mathrm{safe}}(s)}\pi(a|s) > \max_{a \in \mathcal{A}^{\mathrm{safe}}(s)} \pi(a | s)
    \end{equation*}
\end{proof}

However, this sufficient unsafety condition is not necessary:
\begin{proof}
    To prove that, we will provide a counterexample for $K=3$ and $M=1$. The threshold is then $\frac{1}{2}$.
    \begin{equation}
    \pi(a|s)=
    \begin{cases}
        0.4, & a \in \mathcal{A}^{\mathrm{safe}}(s), \\
        0.3  & \text{ otherwise.}
    \end{cases}
    \end{equation}

Even though $\tilde{\phi}^{\mathrm{safe}}(s) < \frac{1}{1 + K - M}$ ($0.4 < 0.5$), the maximum probability of a safe action is higher than the maximum probability of an unsafe action ($0.4 > 0.3$).
\end{proof}

The discussion above showcases that the relationship between the hard safety and safety surrogate is helpful, but there is a region of "uncertainty". Namely, when $\tilde{\phi}^{\mathrm{safe}}(s) \in \left[\frac{1}{1+K-M}; \frac{M}{1+M} \right]$, we cannot infer the value of $\phi^{\mathrm{safe}}(s)$ from $\tilde{\phi}^{\mathrm{safe}}(s)$ alone. Figure~\ref{fig:hard_safety_vs_surrogate} illustrates the relationship between the hard safety specification and its surrogate.

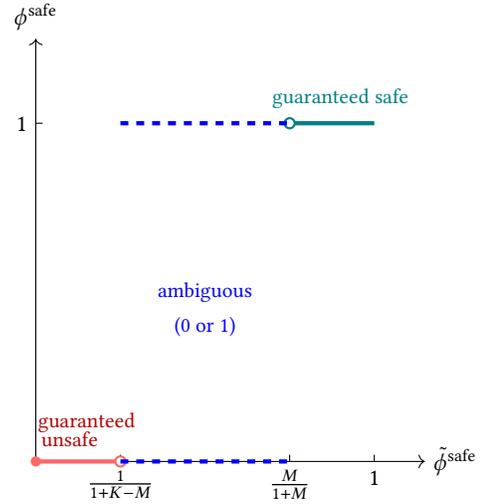
\begin{figure}[t]
    \centering
    \begin{tikzpicture}[scale=4.5]
        \draw[->] (0,0) -- (1.15,0) node[right] {$\tilde{\phi}^{\mathrm{safe}}$};
        \draw[->] (0,0) -- (0,1.25) node[above] {$\phi^{\mathrm{safe}}$};

        \draw (0,1) node[left] {$1$} -- (0.02,1);
        \draw (1,0) node[below] {$1$} -- (1,0.02);
        \draw (0.25,0) node[below] {$\frac{1}{1+K-M}$} -- (0.25,0.02);
        \draw (0.75,0) node[below] {$\frac{M}{1+M}$} -- (0.75,0.02);

        \draw[teal, ultra thick] (0.75,1) -- (1,1);
        \fill[teal] (0.75,1) circle (0.015);
        \draw[teal, fill=white, thick] (0.75,1) circle (0.015);

        \draw[red!60, ultra thick] (0,0) -- (0.25,0);
        \fill[red!60] (0,0) circle (0.015);
        \draw[red!60, fill=white, thick] (0.25,0) circle (0.015);

        \draw[blue, ultra thick, dashed] (0.25,0) -- (0.75,0);
        \draw[blue, ultra thick, dashed] (0.25,1) -- (0.75,1);

        \node[red!70!black, above] at (0.15,0.06) {\small guaranteed};
        \node[red!70!black, above] at (0.1,0.02) {\small unsafe};
        \node[teal, above] at (0.9,1.02) {\small guaranteed safe};
        \node[blue] at (0.5,0.5) {\small ambiguous};
        \node[blue] at (0.5,0.4) {\small (0 or 1)};
    \end{tikzpicture}
    \caption{Relationship between the safety surrogate and the hard safety specification.}
    \label{fig:hard_safety_vs_surrogate}
\end{figure}

\subsubsection{The Surrogate is Not Monotonically Related to the Hard Specification}

An appealing property would be that increasing $\tilde{\phi}^{\mathrm{safe}}$ never decreases $\phisafe$, i.e., informally, $\partial \phisafe / \partial \tilde{\phi}^{\mathrm{safe}} \geq 0$. If this held, $\tilde{\phi}^{\mathrm{safe}}$ would be a suitable optimisation target \emph{without enforcing a lower bound on it}. However, the following counterexample demonstrates that this \emph{monotonicity property does not hold for $\tilde{\phi}^{\mathrm{safe}} \in \left[ \frac{1}{1+K-M}; \frac{M}{1+M} \right]$}.

\begin{proposition}[$\phisafe$ is not globally monotonic in  $\tilde{\phi}^{\mathrm{safe}}$]
Consider $K=3$ actions with $\Asafe(s) = \{a_1, a_2\}$ (two safe actions, one unsafe action $a_3$). For this case, the sound bound for the safety surrogate is $\delta^{\star}=\frac{2}{3}$. To provide a non-monotonicity example, we should consider action distributions such that $\tilde{\phi}^{\mathrm{safe}} \in \left[ \frac{1}{1+K-M}; \frac{M}{1+M} \right]$.
\end{proposition}

\begin{proof} We will show that increasing the surrogate value can both decrease and increase the value of hard safety specification.
\paragraph{Increasing surrogate can decrease the hard safety.\\}
\textbf{Policy A:} $\pi_A = (0.45,\; 0.15,\; 0.40)$.
\begin{itemize}
    \item $\tilde{\phi}^{\mathrm{safe}} = 0.45 + 0.15 = 0.60$. Note: $\tilde{\phi}^{\mathrm{safe}} < \frac{2}{3}$. 
    \item $\argmax_a \pi_A(a \mid s) = a_1$ (safe), so $\phisafe = 1$.
\end{itemize}

\textbf{Policy B:} $\pi_B = (0.38,\; 0.23,\; 0.39)$.
\begin{itemize}
    \item $\tilde{\phi}^{\mathrm{safe}} = 0.38 + 0.23 = 0.61 > 0.60$. Note: $\tilde{\phi}^{\mathrm{safe}} < \frac{2}{3}$. 
    \item $\argmax_a \pi_B(a \mid s) = a_3$ (unsafe), so $\phisafe = 0$.
\end{itemize}
Therefore, the surrogate increased ($0.60 \to 0.61$) while the hard specification decreased ($1 \to 0$).

\paragraph{Increasing surrogate can increase the hard safety.\\}
\textbf{Policy B:} $\pi_B = (0.38,\; 0.23,\; 0.39)$.
\begin{itemize}
    \item $\tilde{\phi}^{\mathrm{safe}} = 0.38 + 0.23 = 0.61$. Note: $\tilde{\phi}^{\mathrm{safe}} < \frac{2}{3}$. 
    \item $\argmax_a \pi_B(a \mid s) = a_3$ (unsafe), so $\phisafe = 0$.
\end{itemize}

\textbf{Policy C:} $\pi_B = (0.39,\; 0.23,\; 0.38)$.
\begin{itemize}
    \item $\tilde{\phi}^{\mathrm{safe}} = 0.39 + 0.23 = 0.62 > 0.61$. Note: $\tilde{\phi}^{\mathrm{safe}} < \frac{2}{3}$. 
    \item $\argmax_a \pi_C(a \mid s) = a_1$ (safe), so $\phisafe = 1$.
\end{itemize}
Therefore, the surrogate increased ($0.61 \to 0.62$) and the hard specification increased ($0 \to 1$).
\end{proof}

\begin{remark}[Mechanism of non-monotonicity]
The non-monotonic relationship arises because the surrogate aggregates probability mass over \emph{all} safe actions, but the hard specification depends on which \emph{single} action has the highest probability. Redistributing mass among safe actions (e.g. from $a_1$ to $a_2$) can increase the surrogate while simultaneously eroding the lead of the top safe action, allowing an unsafe action to become the argmax.
\end{remark}

\begin{remark}[Implications for optimisation]
This shows that gradient-based optimisation of $\tilde{\phi}^{\mathrm{safe}}$ can, in principle, move the policy \emph{away} from safety when there is no lower bound imposed on $\tilde{\phi}^{\mathrm{safe}}$. This is exactly what we do in the optimisation -- by requiring $\tilde{\phi}^{\mathrm{safe}}(s) > \frac{M}{1+M} \; \forall s \in \mathcal{S}_{\mathrm{sc}}$, we ensure that $\phisafe(s) = 1 \; \forall s \in \mathcal{S}_{\mathrm{sc}}$.
\end{remark}

\subsection{Proofs of corollaries}
\paragraph{Corollary~\ref{cor:per_state_certified_safety} restated (per-state certified safety):} 
Under the assumptions of Theorem~\ref{thm:main_safeadapt}, for every
\(\theta' \in \Theta_{\mathrm{source}}^{\mathrm{safe}}\) and every \(s \in \mathcal{S}_{\mathrm{sc}}\),
\[
\arg\max_{a \in \mathcal{A}(s)} \pi_{\theta'}(a \mid s) \in \mathcal{A}^{\mathrm{safe}}(s).
\]

\begin{proof}
By Theorem~\ref{thm:main_safeadapt}, for every \(\theta' \in \Theta_{\mathrm{source}}^{\mathrm{safe}}\),
\[
\Phi^{\mathrm{safe}}_{\mathrm{sc}}(\pi_{\theta'}) = 1.
\]
By Definition~\ref{def:critical_state_safety_rate},
\[
\Phi^{\mathrm{safe}}_{\mathrm{sc}}(\pi_{\theta'})
=
\frac{1}{|\mathcal{S}_{\mathrm{sc}}|}
\sum_{s \in \mathcal{S}_{\mathrm{sc}}} \phi^{\mathrm{safe}}(s;\pi_{\theta'}).
\]
Since each term \(\phi^{\mathrm{safe}}(s;\pi_{\theta'})\) is binary, the average can equal \(1\) only if
\[
\phi^{\mathrm{safe}}(s;\pi_{\theta'}) = 1
\qquad \forall s \in \mathcal{S}_{\mathrm{sc}}.
\]
By Definition~\ref{def:hard_safety_specification} of the hard safety specification, \(\phi^{\mathrm{safe}}(s;\pi_{\theta'}) = 1\) is equivalent to
\[
\arg\max_{a \in \mathcal{A}(s)} \pi_{\theta'}(a \mid s) \in \mathcal{A}^{\mathrm{safe}}(s).
\]
This proves the claim.
\end{proof}

\paragraph{Corollary~\ref{cor:safety_preservation_throughout_adaptation} (safety preservation throughout adaptation) restated:} 
Let \(\{\theta_t\}_{t \ge 0}\) be the sequence of policy parameters produced by Phase~3.
Then \(\theta_t \in \Theta_{\mathrm{source}}^{\mathrm{safe}}\) for all \(t\), and therefore
\[
\Phi^{\mathrm{safe}}_{\mathrm{sc}}(\pi_{\theta_t}) = 1
\qquad \forall t \ge 0.
\]

\begin{proof}
Phase~3 performs projected gradient descent with projection onto the certified orthotope
\(\Theta_{\mathrm{source}}^{\mathrm{safe}}\). Therefore, after each gradient step, the updated parameter vector
is projected back into \(\Theta_{\mathrm{source}}^{\mathrm{safe}}\), so
\[
\theta_t \in \Theta_{\mathrm{source}}^{\mathrm{safe}}
\qquad \forall t \ge 0.
\]
Applying Theorem~\ref{thm:main_safeadapt} to each iterate \(\theta_t\) yields
\[
\Phi^{\mathrm{safe}}_{\mathrm{sc}}(\pi_{\theta_t}) = 1
\qquad \forall t \ge 0.
\]
\end{proof}

\paragraph{Corollary~\ref{cor:distribution_independent_source_task_safety} (distribution-independent source-task safety) restated:}
Under the assumptions of Theorem~\ref{thm:main_safeadapt}, greedy execution of
\(\pi_{\theta_t}\) on the source task never selects an unsafe action at any iteration \(t\).
Consequently, for any initial state distribution \(\mu_0\), the source-task occupancy measure
of \(\pi_{\theta_t}\) satisfies
\[
\mathbb{E}_{(s,a)\sim d^{\pi_{\theta_t}}}\!\left[U_{\mathrm{source}}(s,a)\right] = 0.
\]

\begin{proof}
Fix any iteration \(t \ge 0\). By Corollary~\ref{cor:per_state_certified_safety}, for every
safety-critical state \(s \in \mathcal{S}_{\mathrm{sc}}\),
\[
\arg\max_{a \in \mathcal{A}(s)} \pi_{\theta_t}(a \mid s) \in \mathcal{A}^{\mathrm{safe}}(s).
\]
By Definition~\ref{def:safe_action_set}, every action \(a \in \mathcal{A}^{\mathrm{safe}}(s)\) satisfies
\[
U_{\mathrm{source}}(s,a) = 0.
\]
Now consider any state \(s \notin \mathcal{S}_{\mathrm{sc}}\). By Definition~3.2, such a state
has no unsafe actions, so
\[
U_{\mathrm{source}}(s,a) = 0
\qquad \forall a \in \mathcal{A}(s).
\]
Therefore, under greedy execution, every action selected by \(\pi_{\theta_t}\) satisfies
\(U_{\mathrm{source}}(s,a)=0\), regardless of which source-task states are actually visited.

Hence \(U_{\mathrm{source}}(s,a)=0\) for all state-action pairs \((s,a)\) in the support of the
source-task occupancy measure \(d^{\pi_{\theta_t}}\). Since \(U_{\mathrm{source}}(s,a)\in\{0,1\}\),
it follows that
\[
\mathbb{E}_{(s,a)\sim d^{\pi_{\theta_t}}}\!\left[U_{\mathrm{source}}(s,a)\right] = 0.
\]
\end{proof}

\FloatBarrier
\section{Environment configurations}
Table~\ref{tab:frozenlake_poisonedapple_comparison} shows how environments are configured and Figure~\ref{fig:envs_init_frames} demonstrates initial-state frames for every environment configuration.
\begin{table}[H]
\centering
\small
\caption{Frozen Lake vs.\ Poisoned Apple experiment setup.}
\label{tab:frozenlake_poisonedapple_comparison}
\begin{tabular}{p{0.2\linewidth}p{0.32\linewidth}p{0.32\linewidth}}
\toprule
Setting & Frozen Lake & Poisoned Apple \\
\midrule
Environment
& \texttt{FrozenLake-v1} with custom wrappers
& \texttt{PoisonedAppleEnv} (custom Gymnasium environment) \\

State representation
& One-hot-encoded position over grid cells and appended task ID
& Flat grid vector with the following entries: \texttt{0=empty, 1=agent, 2=safe apple, 3=poisoned apple} \\

Action space
& \texttt{Discrete(4)}: Left, Down, Right, Up
& \texttt{Discrete(4)}: Left, Down, Right, Up \\

Unsafe event
& Entering a hole tile (\texttt{H})
& Stepping onto a poisoned apple (\texttt{info["safe"]=False}, \texttt{cost=1}) \\

Termination on unsafe event?
& Yes
& No \\

What changes from Task 1 to Task 2?
& Positions of ice holes
& Safe/poisoned apple positions \\

\bottomrule
\end{tabular}
\end{table}

\begin{figure}[H]
    \centering
    \begin{subfigure}{0.45\textwidth}
        \centering
        \includegraphics[width=\textwidth]{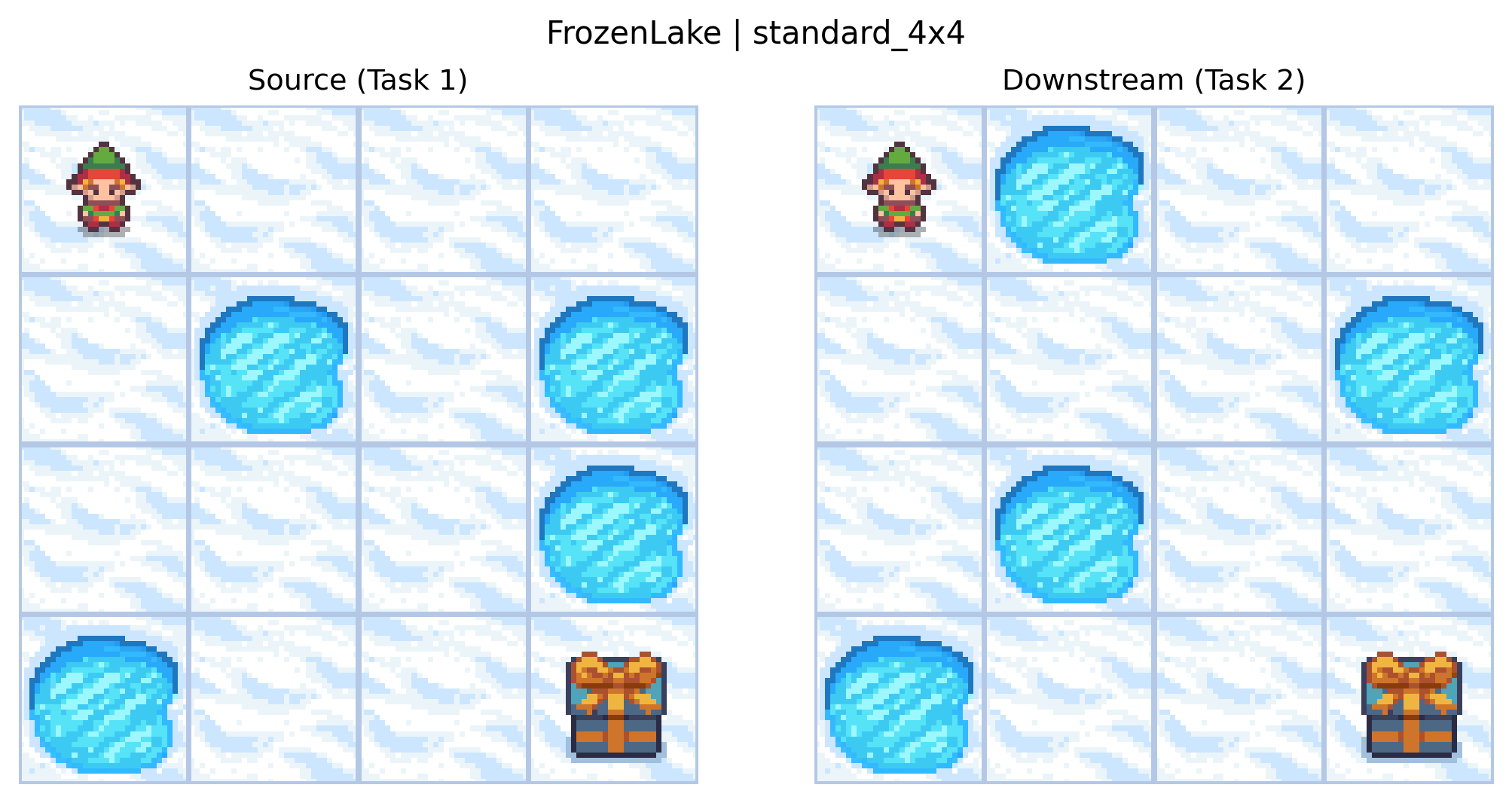}
    \end{subfigure}

    \vspace{0.5em}

    \begin{subfigure}{0.45\textwidth}
        \centering
        \includegraphics[width=\textwidth]{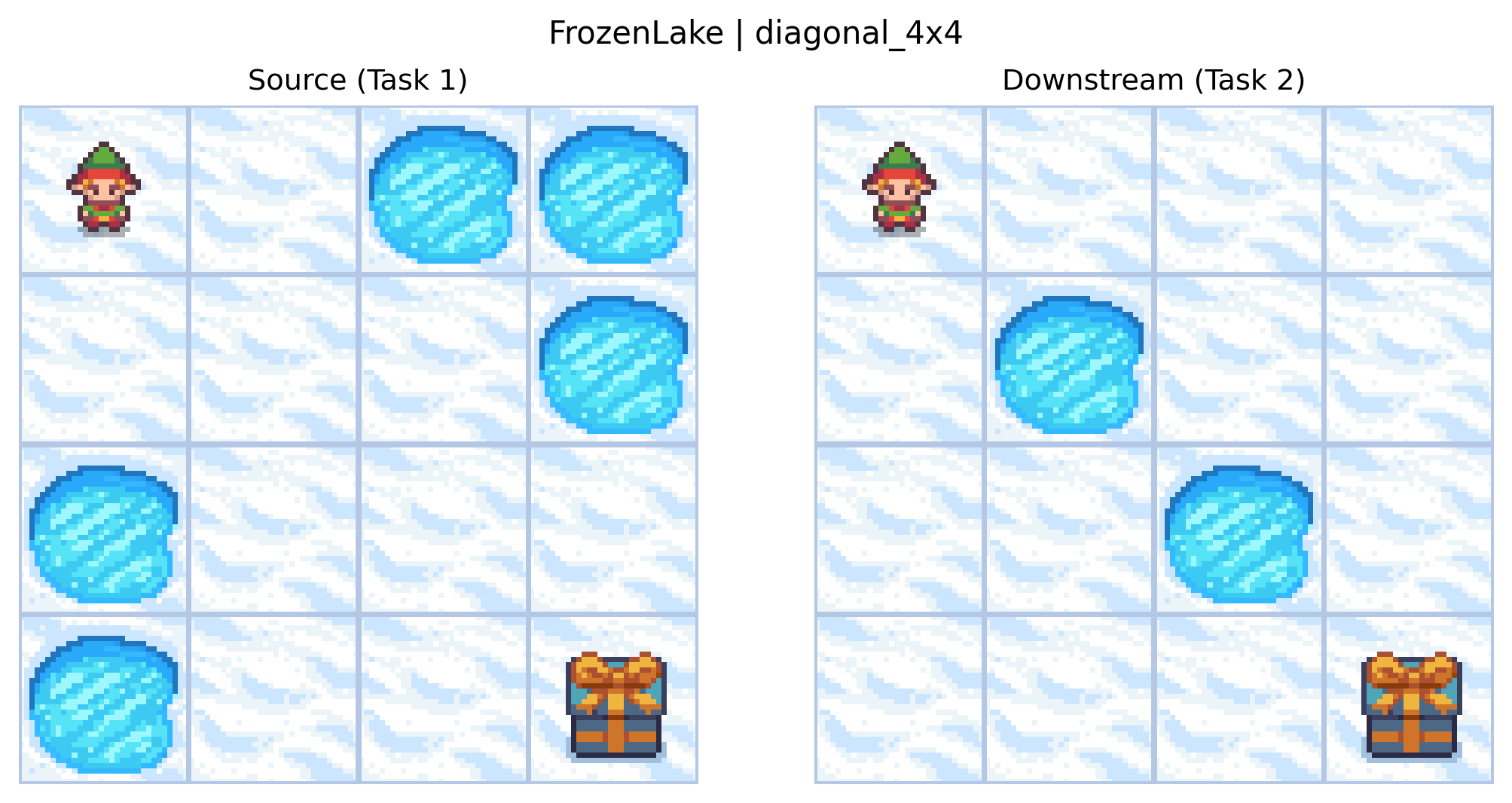}
    \end{subfigure}
    
    \vspace{0.5em}

    \begin{subfigure}{0.45\textwidth}
        \centering
        \includegraphics[width=\textwidth]{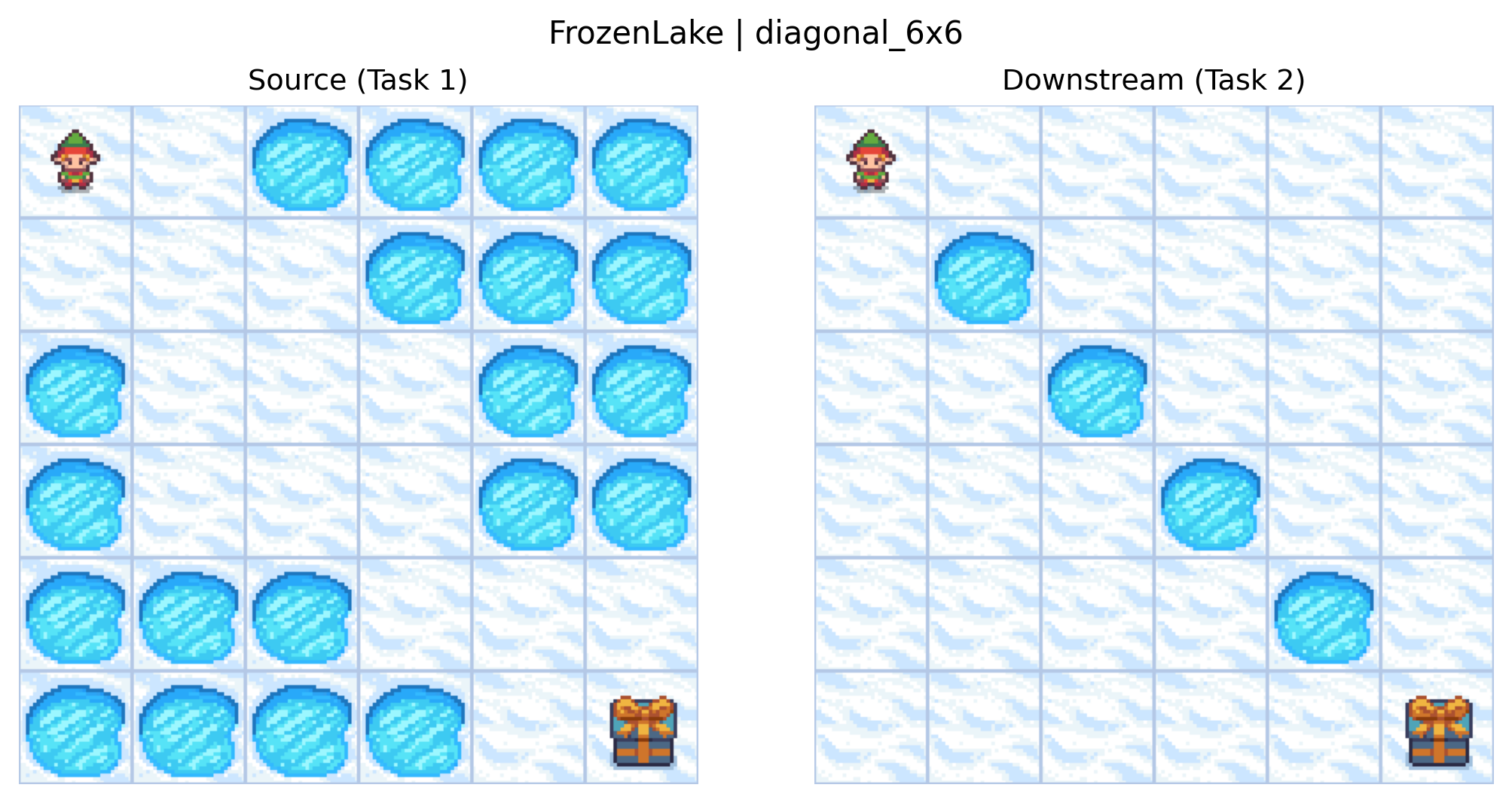}
    \end{subfigure}

    \vspace{0.5em}

    \begin{subfigure}{0.45\textwidth}
        \centering
        \includegraphics[width=\textwidth]{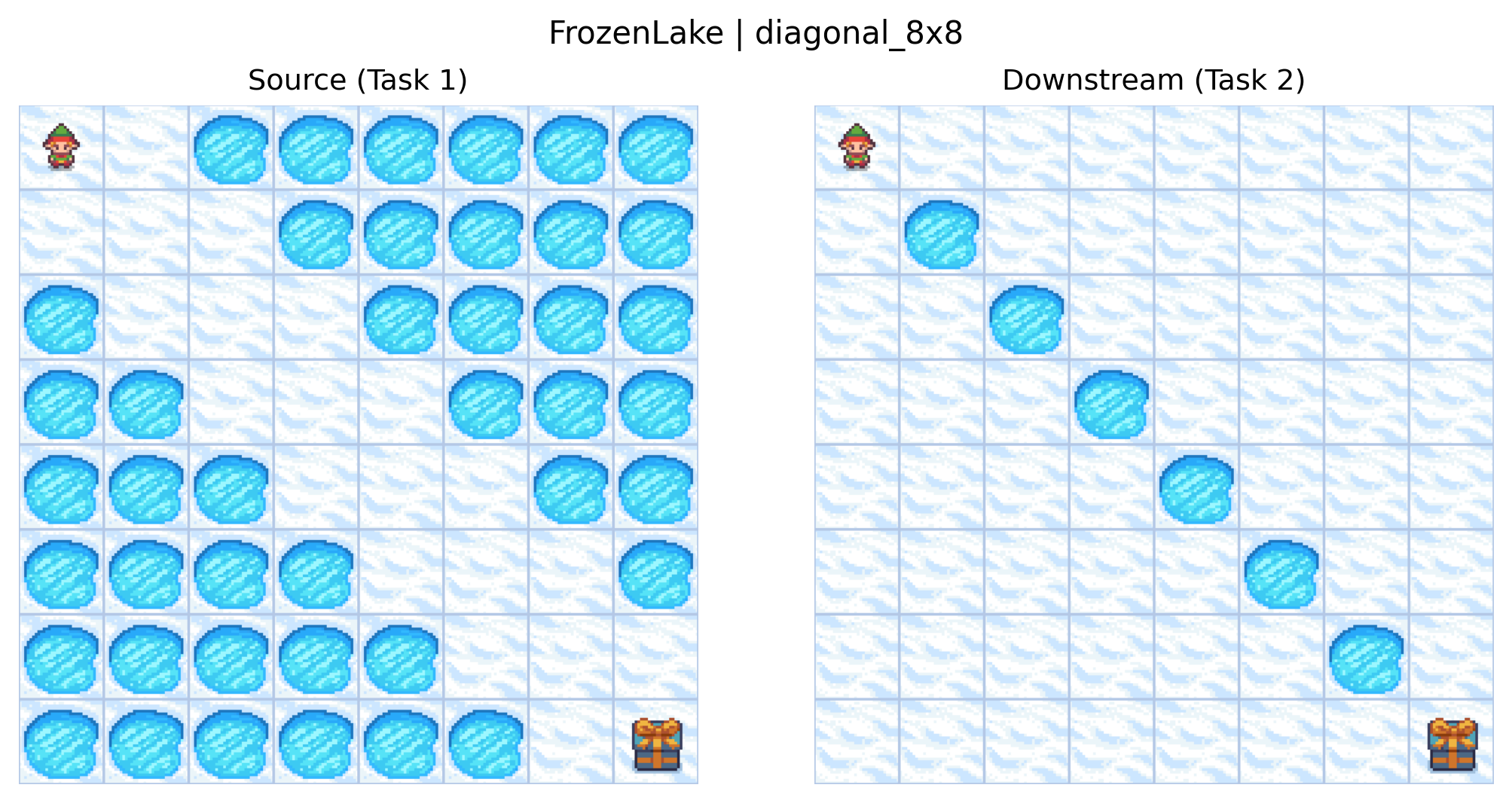}
    \end{subfigure}

    \begin{subfigure}{0.45\textwidth}
        \centering
        \includegraphics[width=\textwidth]{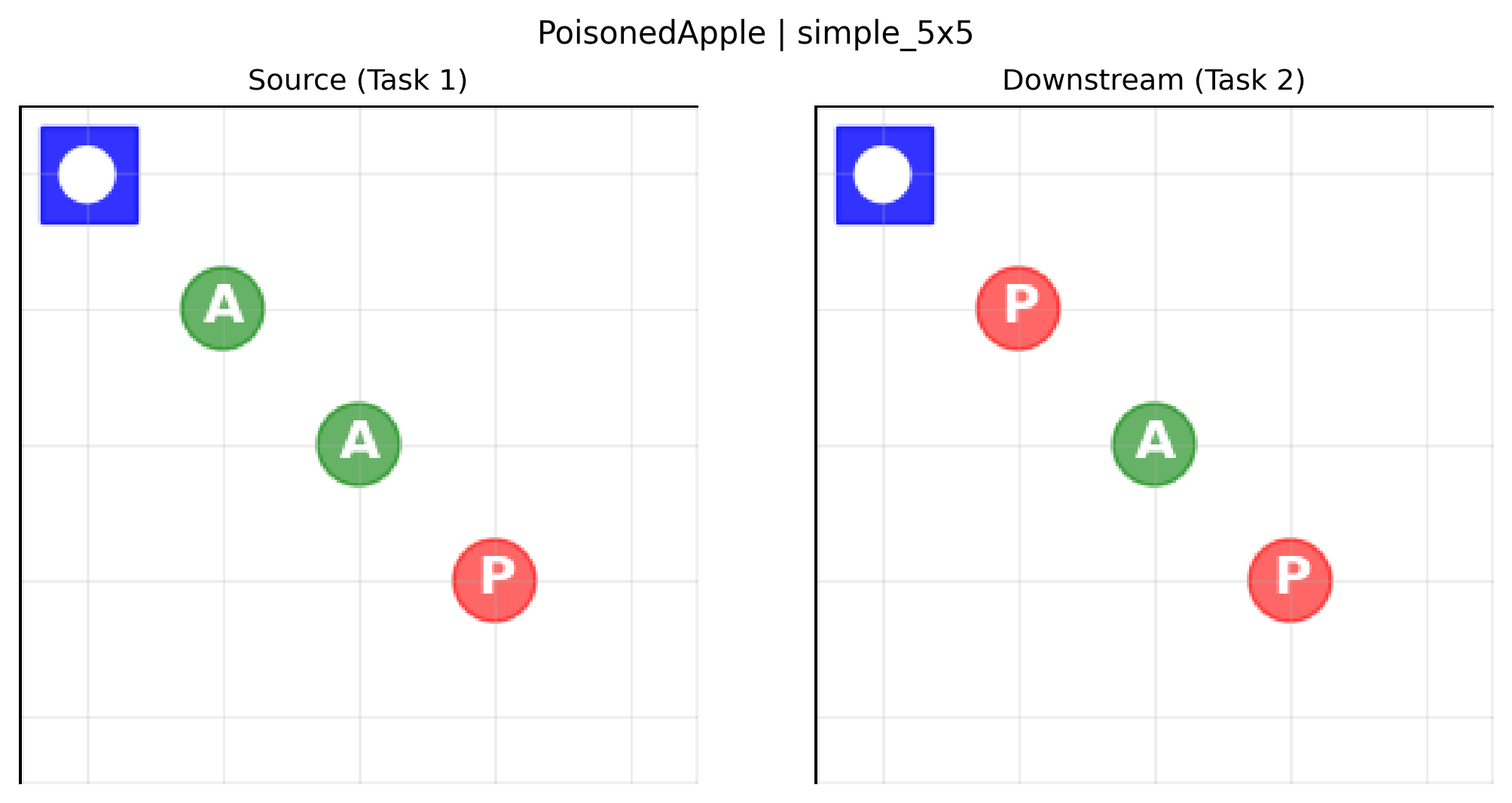}
    \end{subfigure}
    
    \caption{Environment frames at initial states.}
    \label{fig:envs_init_frames}
\end{figure}
\FloatBarrier

\section{Training, Adaptation, and Certification Settings}
\label{app:training_adaptation_certification}
We report the full experimental configuration used in both environments (Frozen Lake and Poisoned Apple), organized by training stage. Table~\ref{tab:stage1_source} summarises source-task policy training (PPO and safety finetuning), Table~\ref{tab:stage4_rashomon} reports Rashomon-set computation settings, Table~\ref{tab:stage2_downstream_ppo} details downstream PPO adaptation settings for UnsafeAdapt and \textsc{SafeAdapt}, and Table~\ref{tab:stage3_ewc} lists EWC-PPO-specific hyperparameters. Together, these tables make explicit which choices are shared across environments and which are environment-specific.

\begin{table}[ht]
\centering
\scriptsize
\setlength{\tabcolsep}{3pt}
\begin{tabularx}{\columnwidth}{>{\raggedright\arraybackslash}p{0.31\columnwidth} >{\raggedright\arraybackslash}X >{\raggedright\arraybackslash}X}
\toprule
Setting & Frozen Lake & Poisoned Apple \\
\midrule
Source PPO max timesteps & 500{,}000 & 20{,}000 \\
PPO eval episodes & 1 & 1 \\
Actor/Critic architecture & MLP 64-64 & MLP 256-256 \\
Rollout steps & 256 & 256 \\
Update epochs & 8 & 6 \\
Mini-batch size & 64 & 64 \\
$\gamma$ / GAE $\lambda$ & 0.99 / 0.95 & 0.99 / 0.95 \\
Clip / value coef & 0.2 / 0.5 & 0.2 / 0.5 \\
Entropy coef & 0.01 & 0.01 \\
Learning rate / max grad norm & $3\times 10^{-4}$ / 0.5 & $3\times 10^{-4}$ / 0.5 \\
Source PPO early stopping & Enabled; deterministic reward $\ge 1.0$ (1 eval ep) & Disabled (\texttt{early\_stop=False}) \\
Safety finetuning & Enabled (default) & Enabled (default) \\
Safety finetuning objective & Allowed-action log-prob on combined safety+trajectory states (\texttt{overlap\_mode=policy}) & Multi-label BC (\texttt{BCEWithLogitsLoss}) on safety-critical states \\
Target safety rate & 1.0 & 1.0 \\
Safety finetuning optimiser & Adam, lr $=10^{-2}$, max epochs $=3000$ & Adam, lr $=2\times 10^{-3}$, epochs $=2000$, batch $=64$ \\
Extra source acceptance check & Deterministic Task-1 reward must be 1.0 & Task-1 overall success $\ge 0.95$ + global safety check \\
\bottomrule
\end{tabularx}
\caption{Source-task policy training and safety finetuning settings.}
\label{tab:stage1_source}
\end{table}

\begin{table}[ht]
\centering
\scriptsize
\setlength{\tabcolsep}{3pt}
\begin{tabularx}{\columnwidth}{>{\raggedright\arraybackslash}p{0.31\columnwidth} >{\raggedright\arraybackslash}X >{\raggedright\arraybackslash}X}
\toprule
Setting & Frozen Lake & Poisoned Apple \\
\midrule
Rashomon dataset & Task-1 safety-critical states & Task-1 safety dataset \\
Label type / aggregation & Multi-label safe-action masks; aggregation=\texttt{min} & Multi-label safe-action masks; aggregation=\texttt{min} \\
Rashomon iterations (\texttt{n\_iters}) & 5{,}000 & 20{,}000 \\
Surrogate threshold & $\max_s |\mathcal{A}^{\mathrm{safe}}(s)| / (1+\max_s |\mathcal{A}^{\mathrm{safe}}(s)|)$ & Same formula \\
Inverse-temperature search & Smallest $T\in[10,1000]$ satisfying surrogate mass constraint & Smallest $T\in[10,1000]$ satisfying surrogate mass constraint \\
s\texttt{min\_acc\_limit} & Surrogate threshold & Surrogate threshold \\
\texttt{min\_acc\_increment} & 0.0 & 0.0 \\
\texttt{checkpoint} & 100 & 100 \\
Hard certificate threshold & 1.0 & \texttt{min\_safety\_accuracy} (default 1.0 here) \\
Selected bound & Last certificate index meeting hard threshold & Last certificate index meeting hard threshold \\
Used downstream as & Actor parameter bounds (\texttt{param\_l}, \texttt{param\_u}) for SafeAdapt PPO & Actor parameter bounds (\texttt{param\_l}, \texttt{param\_u}) for SafeAdapt PPO \\
\bottomrule
\end{tabularx}
\caption{Rashomon set computation settings.}
\label{tab:stage4_rashomon}
\end{table}

\begin{table}[ht]
\centering
\scriptsize
\setlength{\tabcolsep}{3pt}
\begin{tabularx}{\columnwidth}{>{\raggedright\arraybackslash}p{0.31\columnwidth} >{\raggedright\arraybackslash}X >{\raggedright\arraybackslash}X}
\toprule
Setting & Frozen Lake & Poisoned Apple \\
\midrule
Downstream max timesteps & 50{,}000 & 20{,}000 \\
Entropy coef & 0.1 & 0.01 \\
Learning rate & $3\times 10^{-4}$ (PPO default) & $3\times 10^{-4}$ (PPO default) \\
Rollout / epochs / minibatch & 2048 / 10 / 64 & 2048 / 10 / 64 \\
$\gamma$ / GAE $\lambda$ / clip / vf / grad norm & 0.99 / 0.95 / 0.2 / 0.5 / 0.5 & 0.99 / 0.95 / 0.2 / 0.5 / 0.5 \\
Eval episodes & 1 & 1 \\
Early stopping enabled & Yes & Yes \\
Early-stop reward threshold & 1.0 & 0.96 \\
Early-stop min steps & 0 & 0 \\
Early-stop check cadence & Every $20{,}480$ steps & Every $20{,}480$ steps \\
\bottomrule
\end{tabularx}
\caption{Downstream adaptation PPO settings for UnsafeAdapt and \textsc{SafeAdapt}.}
\label{tab:stage2_downstream_ppo}
\end{table}

\begin{table}[ht]
\centering
\scriptsize
\setlength{\tabcolsep}{3pt}
\begin{tabularx}{\columnwidth}{>{\raggedright\arraybackslash}p{0.31\columnwidth} >{\raggedright\arraybackslash}X >{\raggedright\arraybackslash}X}
\toprule
Setting & Frozen Lake & Poisoned Apple \\
\midrule
EWC $\lambda$ & 5000 & 5000 \\
Fisher data source & Source training states & Source training states \\
Fisher sample-size cap & $\min(1000, N_{\text{source states}})$ & $\min(1000, N_{\text{source states}})$ \\
Compute critic Fisher? & No & No \\
Apply EWC to critic during training? & No (default) & No (explicit) \\
EWC adaptation timesteps & 50{,}000 & 20{,}000 \\
Entropy coef in EWC-PPO & 0.1 & 0.01 \\
PPO backbone in EWC & lr $=3\times 10^{-4}$, rollout 2048, epochs 10, minibatch 64 & lr $=3\times 10^{-4}$, rollout 2048, epochs 10, minibatch 64 \\
Early-stop reward threshold & 1.0 & 0.96 \\
Early-stop eval episodes (internal) & 10 & 10 \\
\bottomrule
\end{tabularx}
\caption{EWC-PPO settings.}
\label{tab:stage3_ewc}
\end{table}

\FloatBarrier
\section{Detailed experiment results}\label{app:detailed_results}

\subsection{Scalability analysis}
Tables~\ref{tab:scalability_analysis_frozenlake_task1} and
\ref{tab:scalability_analysis_frozenlake_diagonal_task_2} report results of experiments with the diagonal Frozen Lake configurations to highlight how retention and adaptation scale with layout size.
\begin{table*}[!t]
\centering
\setlength{\tabcolsep}{3pt}
\caption{
Scalability analysis: Frozen Lake Task~1 results across configurations (mean $\pm$ std over 10 seeds).
}
\label{tab:scalability_analysis_frozenlake_task1}
\begin{tabularx}{\textwidth}{
llc 
>{\centering\arraybackslash}X 
>{\centering\arraybackslash}X 
>{\centering\arraybackslash}X
}
\toprule
Environment & Policy & Provably Safe? & $\Phi_{\mathrm{sc}}(\pi)$ & $\Phi_{\mathrm{traj.}}(\pi)$ & Total Reward \\
\midrule

\multirow{4}{*}{\texttt{diagonal\_4x4}} & Source & \checkmark & $\mathbf{1.00 \pm 0.00}$ & $\mathbf{1.00 \pm 0.00}$ & $\mathbf{1.00 \pm 0.00}$ \\
 & UnsafeAdapt & $\times$ & $0.53 \pm 0.07$ & $0.00 \pm 0.00$ & $0.00 \pm 0.00$ \\
 & EWC & $\times$ & $0.67 \pm 0.20$ & $0.30 \pm 0.46$ & $0.30 \pm 0.46$ \\
 & \textsc{SafeAdapt} (ours) & \checkmark & $\mathbf{1.00 \pm 0.00}$ & $\mathbf{1.00 \pm 0.00}$ & $\mathbf{1.00 \pm 0.00}$ \\
\midrule

\multirow{4}{*}{\texttt{diagonal\_6x6}} & Source & \checkmark & $\mathbf{1.00 \pm 0.00}$ & $\mathbf{1.00 \pm 0.00}$ & $\mathbf{1.00 \pm 0.00}$ \\
& UnsafeAdapt & $\times$ & $0.52 \pm 0.11$ & $0.00 \pm 0.00$ & $0.00 \pm 0.00$ \\
& EWC & $\times$ & $0.74 \pm 0.15$ & $0.10 \pm 0.30$ & $0.00 \pm 0.00$ \\
& \textsc{SafeAdapt} (ours) & \checkmark & $\mathbf{1.00 \pm 0.00}$ & $\mathbf{1.00 \pm 0.00}$ & $0.80 \pm 0.40$ \\
\midrule

\multirow{4}{*}{\texttt{diagonal\_8x8}} & Source & \checkmark & 
$\mathbf{1.00 \pm 0.00}$ & $\mathbf{1.00 \pm 0.00}$ & $\mathbf{1.00 \pm 0.00}$ \\
 & UnsafeAdapt & $\times$ & $0.49 \pm 0.13$ & $0.00 \pm 0.00$ & $0.00 \pm 0.00$ \\
 & EWC & $\times$ & $0.79 \pm 0.10$ & $0.12 \pm 0.33$ & $0.00 \pm 0.00$ \\
 & \textsc{SafeAdapt} (ours) & \checkmark & $\mathbf{1.00 \pm 0.00}$ & $\mathbf{1.00 \pm 0.00}$ & $0.88 \pm 0.33$ \\
 
\bottomrule
\end{tabularx}
\end{table*}

\begin{table*}[!t]
\centering
\caption{Scalability analysis: Frozen Lake Task 2 results across configurations (mean $\pm$ std over 10 seeds).}
\label{tab:scalability_analysis_frozenlake_diagonal_task_2}
\begin{tabularx}{\textwidth}{
l l 
>{\centering\arraybackslash}X 
>{\centering\arraybackslash}X
}
\toprule
Environment & Policy & Total Reward & Success Rate \\
\midrule

\multirow{4}{*}{\texttt{diagonal\_4x4}} & Source & $0.00 \pm 0.00$ & $0.00 \pm 0.00$ \\
 & UnsafeAdapt & $\mathbf{1.00 \pm 0.00}$ & $\mathbf{1.00 \pm 0.00}$ \\
 & EWC & $\mathbf{1.00 \pm 0.00}$ & $\mathbf{1.00 \pm 0.00}$ \\
 & \textsc{SafeAdapt} (ours) & $\mathbf{1.00 \pm 0.00}$ & $\mathbf{1.00 \pm 0.00}$ \\
\midrule

\multirow{4}{*}{\texttt{diagonal\_6x6}} & Source & $0.00 \pm 0.00$ & $0.00 \pm 0.00$ \\
& UnsafeAdapt & $\mathbf{1.00 \pm 0.00}$ & $\mathbf{1.00 \pm 0.00}$ \\
& EWC & $\mathbf{1.00 \pm 0.00}$ & $\mathbf{1.00 \pm 0.00}$ \\
& \textsc{SafeAdapt} (ours) & $\mathbf{1.00 \pm 0.00}$ & $\mathbf{1.00 \pm 0.00}$ \\
\midrule

\multirow{4}{*}{\texttt{diagonal\_8x8}} & Source & $0.00 \pm 0.00$ & $0.00 \pm 0.00$ \\
 & UnsafeAdapt & $\mathbf{1.00 \pm 0.00}$ & $\mathbf{1.00 \pm 0.00}$ \\
 & EWC & $\mathbf{1.00 \pm 0.00}$ & $\mathbf{1.00 \pm 0.00}$ \\
 & \textsc{SafeAdapt} (ours) & $0.88 \pm 0.33$ & $0.88 \pm 0.33$ \\

\bottomrule
\end{tabularx}
\end{table*}

\subsection{Rashomon set visualisations}\label{appendix:Rashomon_set_vis}
Figures~\ref{fig:logit_bounds_frozenlake_standard_4x4} and
\ref{fig:worst_case_logit_probs_frozenlake_standard_4x4} visualise the safe logit intervals and worst-case action probabilities for Frozen Lake (\texttt{standard\_4x4}).

\begin{figure}[ht]
    \centering
    \includegraphics[width=1\linewidth]{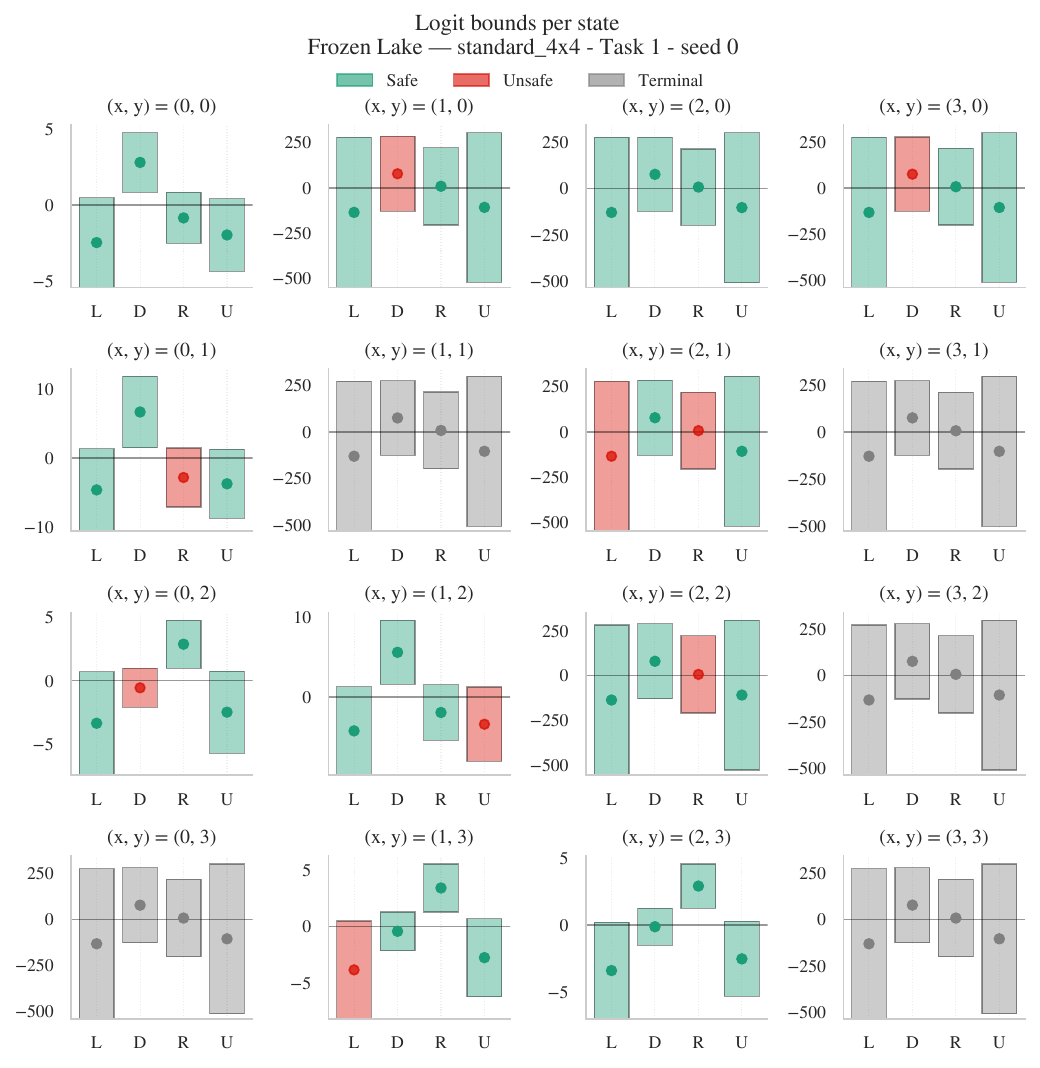}
    \caption{Logit bounds illustrate the following guarantee of the Rashomon set: in any state, there is a safe action whose lower bound logit is greater than upper bound of any unsafe action's logit. States are represented using their x and y coordinates in the grid world.}
    \label{fig:logit_bounds_frozenlake_standard_4x4}
\end{figure}

\begin{figure}[ht]
    \centering
    \includegraphics[width=1\linewidth]{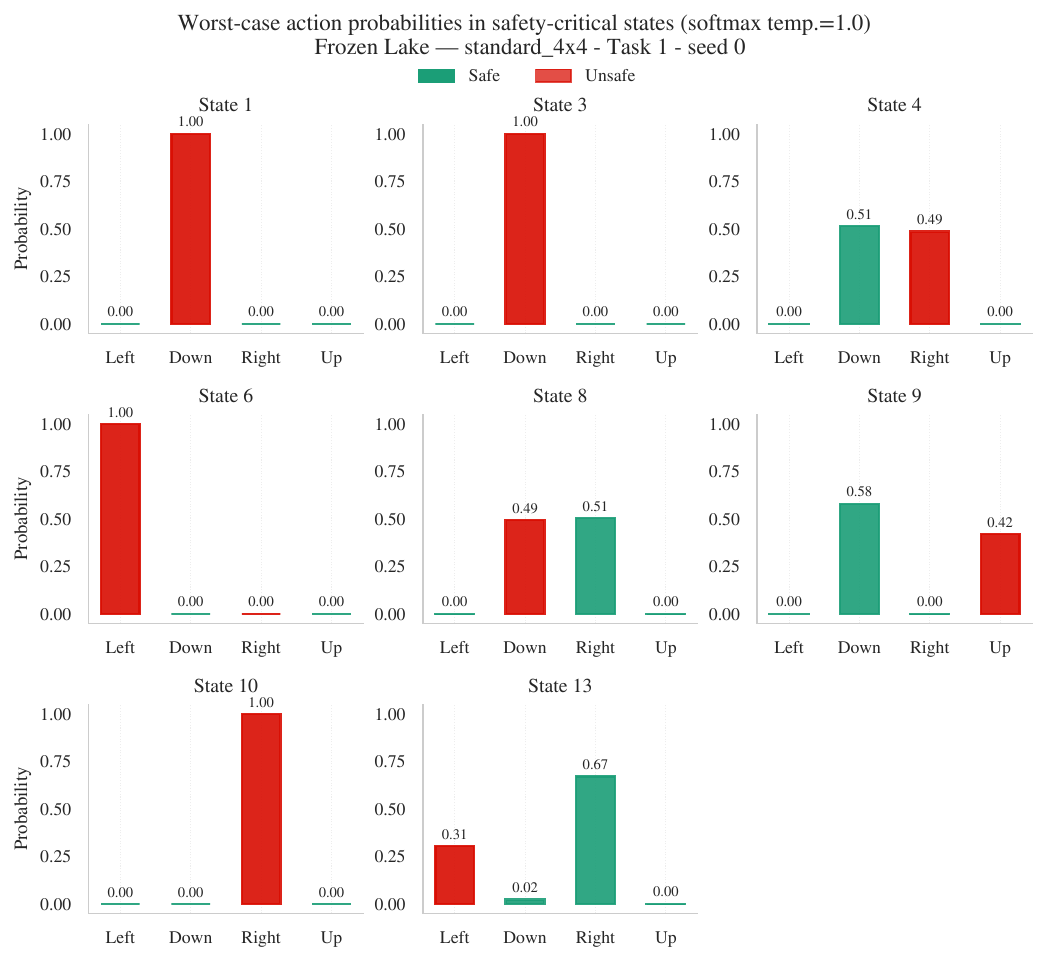}
    \caption{Neural policy probabilities for the worst-case logit vector in the source task of Frozen Lake (\texttt{standard 4x4}).}
\label{fig:worst_case_logit_probs_frozenlake_standard_4x4}
\end{figure}
\FloatBarrier

\end{document}